\documentclass[journal]{IEEEtran}
\usepackage{amsfonts}
\usepackage{cite,graphicx,amsmath,amsthm}
\usepackage{fancyhdr}
\usepackage{dsfont}
\usepackage{array,color}
\usepackage{bm}
\usepackage{float}
\usepackage{algpseudocode}
\usepackage{multirow}
\usepackage{booktabs}
\usepackage{makecell}
\usepackage{hyperref}
\usepackage{graphicx}
\usepackage{pifont}
\usepackage{amsmath,amssymb, graphicx, epstopdf,cite,enumerate,booktabs, setspace, float,stfloats,bm, multirow, lscape, caption, subcaption, graphbox, soul, color, xcolor, hyperref}
\usepackage{diagbox}
\hypersetup{hidelinks}
\usepackage{cite, amsmath,amssymb, bm}
\usepackage{bbding}
\usepackage{amsthm}
\usepackage[normalem]{ulem}
\usepackage[linesnumbered,ruled]{algorithm2e}

\allowdisplaybreaks[4]

\newtheorem{lemma}{Lemma}

\newtheorem{proposition}{Proposition}

\newtheorem{corollary}{Corollary}

\newtheorem{property}{Property}

\newtheorem{remark}{Remark}

\newtheorem{claim}{Claim}

\title{
STT-GS: Sample-Then-Transmit\\Edge Gaussian Splatting with Joint Client Selection and Power Control}

\author{Zhen Li, Xibin Jin, Guoliang Li, Shuai Wang, Miaowen Wen, Huseyin Arslan,~\emph{Fellow, IEEE},\\Derrick Wing Kwan Ng,~\emph{Fellow, IEEE}, and 
Chengzhong Xu,~\emph{Fellow, IEEE}
\thanks{
This work was supported by the National Natural Science Foundation of China (Grant No. 62371444), and the Shenzhen Science and Technology Program (Grant No. RCYX20231211090206005, JCYJ20241202124934046). 
Corresponding authors: Guoliang Li and Shuai Wang.

Zhen Li is with School of Automation, Hangzhou Dianzi University, Hangzhou, China. (lizhen@hdu.edu.cn)

Xibin Jin and Miaowen Wen are with School of Electronic and Information Engineering, South China University of Technology, Guangzhou, China. (eexbj@mail.scut.edu.cn; eemwwen@scut.edu.cn)

Shuai Wang is with Shenzhen Institutes of Advanced Technology, Chinese Academy of Sciences, Shenzhen, China. (s.wang@siat.ac.cn)

Guoliang Li and Chengzhong Xu are with Department of Computer and Information Science, University of Macau, Macau, China. (e-mail: li.guoliang@connect.um.edu.mo; czxu@um.edu.mo)

Huseyin Arslan is with College of Engineering, Istanbul Medipol University, Istanbul, Turkey. (huseyinarslan@medipol.edu.tr)

Derrick Wing Kwan Ng is with School of Electrical Engineering and Telecommunications, the University of New South Wales, Sydney, Australia. (w.k.ng@unsw.edu.au)}
}

\begin{document}

\maketitle

\begin{abstract}
Edge Gaussian splatting (EGS), which aggregates data from distributed clients (e.g., drones) and trains a global GS model at the edge (e.g., ground server), is an emerging paradigm for scene reconstruction in low-altitude economy. Unlike traditional edge resource management methods that emphasize communication throughput or general-purpose learning performance, EGS explicitly aims to maximize the GS qualities, rendering existing approaches inapplicable. 
To address this problem, this paper formulates a novel GS-oriented objective function that distinguishes the heterogeneous view contributions of different clients. However, evaluating this function in turn requires clients' images, leading to a causality dilemma. 
To this end, this paper further proposes a sample-then-transmit EGS (or STT-GS for short) strategy, which first samples a subset of images as pilot data from each client for loss prediction. Based on the first-stage evaluation, communication resources are then prioritized towards more valuable clients. To achieve efficient sampling, a feature-domain clustering (FDC) scheme is proposed to select the most representative data and pilot transmission time minimization (PTTM) is adopted to reduce the pilot overhead.
Subsequently, we develop a joint client selection and power control (JCSPC) framework to maximize the GS-oriented function under communication resource constraints. Despite the nonconvexity of the problem, we propose a low-complexity efficient solution based on the penalty alternating majorization minimization (PAMM) algorithm. 
Experiments reveal that the proposed scheme significantly outperforms existing benchmarks on real-world datasets. 
The GS-oriented objective can be accurately predicted with low sampling ratios (e.g., $10\,\%$), and our method achieves an excellent tradeoff between view contributions and communication costs.
\end{abstract}


\section{Introduction}

Reconstructing three-dimensional (3D) environments is crucial for embodied robotics and low-altitude economy (LAE) applications \cite{agarwal2011building, zhu2018very,fruh2004automated,xu20193d,li2024seamless}.
However, conventional 3D reconstruction methods rely on neural radiance fields (NeRF), which often incur considerable computational overhead \cite{mildenhall2021nerf}. To address this limitation, 3D Gaussian splatting (GS)\cite{kerbl20233d} has recently been proposed through providing enhanced computational efficiency through explicit 3D geometric representations. 
For instance, 3D GS can be trained within ten minutes and is capable of rendering images at over 50\,Hz \cite{jena2025sparfels}.
In practice, reconstructing large-scale scenes inevitably requires leveraging data distributed across multiple robots \cite{maboudi2023review}. 
This motivates the adoption of the edge GS (EGS) paradigm, which aggregates data from distributed clients and trains a global GS model at the edge server.

{\color{black}
\subsection{Related Work}

Generally, edge resource optimization relies heavily on the designed objective functions. Conventional approaches \cite{yu2004iterative,zhang2024efficient,schubert2004solution,al2011achieving,beck20141,wang2020learning,wang2020machine,yoo2019learning} adopt communication throughput \cite{yu2004iterative,zhang2024efficient,schubert2004solution}, user fairness \cite{zheng2016wireless}, power consumption \cite{schubert2004solution,zhang2024integrated}, or general-purpose learning performance \cite{wang2020learning,wang2020machine,yoo2019learning,luo2023joint,luo2024deep} as their optimization objectives. 
Nonetheless, these metrics are not tailored for GS and cannot explicitly maximize the actual rendering quality. 
Recent studies attempt to design loss functions tailored to visual tasks, such as adjusting transmission priorities via feature matching errors \cite{wang2024mba}. However, these methods fail to account for the non-uniform contribution of Gaussian primitives to the eventual rendering quality. There also exist distributed GS methods that account for the rendering loss based on gradient importance \cite{suzuki2024fed3dgs}. Yet, such importance evaluation is often based on heuristics (e.g., gradient norm thresholding) and cannot fully reflect the geometric data characteristics. 
Therefore, how to define a proper objective function for EGS still remains an open question.

Given a certain objective function, the next step is to conduct optimization. 
Conventional edge resource optimization mainly focuses on physical-layer aspects. For instance, the learning centric power allocation method is proposed in \cite{wang2020learning}, and the multi-antenna beamforming design is proposed in \cite{liang2024communication} for edge intelligence systems.
To achieve the best rendering performance, there is a paradigm shift towards cross-layer optimization \cite{marchisio2019deep,zhou2019edge,feng2024edge,duan2023combining}. 
In this direction, prior works focus on edge video streaming \cite{feng2024edge} or edge machine learning \cite{zhou2019edge,duan2023combining}, improving general-purpose multimedia or learning performance by jointly optimizing physical-layer (e.g., power allocation) and application-layer (e.g., task scheduling) parameters. 
Nevertheless, these approaches neglect the GS rendering requirements.
Indeed, the clients should be prioritized according to their potential contributions to multi-view GS rendering, while taking their channel conditions into account. 
Such cross-layer GS optimization has not been addressed in existing literature.

\subsection{Contributions of This Work}

To fill the research gap, this paper formulates a novel GS-oriented objective function that distinguishes the heterogeneous view contributions of different clients.
This objective function, which is built upon the uncertainty sampling theory~\cite{yoo2019learning, zhu2009active, beluch2018power} and the vanilla GS loss function \cite{kerbl20233d}, enables the edge server to maximize the information gain brought by the collected GS data. 
Subsequently, we integrate the new objective function with the communication conditions for effective GS cross-layer optimization.  
This leads to a joint client selection and power control (JCSPC) framework, which unifies viewpoint contribution, interference mitigation, and resource allocation into a single optimization formulation. 

However, the above GS problem involves two key technical challenges. First, the GS-oriented objective, being a function of clients’ images, is inaccessible prior to data transmissions, while data transmissions conversely require the guidance from the GS-oriented objective.
This thereby introduces a so-called \textbf{causality dilemma} (i.e., chicken-and-egg paradox). 
Second, JCSPC involves nonlinear coupling among binary client selection variables and continuous power control variables, due to potential cross-layer multi-user interference.
Consequently, it belongs to a \textbf{nonconvex mixed integer nonlinear programming (NMINLP)} problem, which is nonconvex even after continuous relaxation. 
Existing approaches, e.g., convex optimization \cite{diamond2016cvxpy,yu2004iterative,schubert2004solution} or majorization minimization (MM) \cite{sun2016majorization,wang2020learning,wen2023byzantine,wen2024augment}, are not applicable.

To address the first challenge, this paper proposes a sample-then-transmit EGS (STT-GS for short) strategy, which first samples a subset of images as pilot data from each client
for loss prediction and then prioritizes resources towards more valuable clients based on the first-stage evaluation. 
Since sampling also involves communication costs, it is necessary to improve the sampling efficiency and reduce the pilot overhead. 
Thus, we propose feature-domain clustering (FDC) to select the most representative data, where the sampling ratio of FDC is determined using cross-validation.
Subsequently, we propose a fast iterative bisection search algorithm for pilot transmission time minimization (PTTM). 
To tackle the second challenge, we propose a penalty alternating majorization minimization (PAMM) method to solve the JCSPC problem.
Our PAMM first adopts variable splitting to decouple the selection and power variables involved in constraints, and leverages penalization to ensure constraint feasibility. 
Such reformulation yields a separable structure that can be safely handled by alternating minimization (AM), which is guaranteed to converge to a local minimum. \cite{beck20141}. 
By further incorporating MM into AM for joint optimization, the proposed PAMM algorithm effectively solves the NMINLP problem with polynomial computational complexity. 

We evaluate the proposed STT-GS scheme along with the FDC, PTTM, PAMM algorithms exploiting two \emph{real-world drone gathered} LAE datasets, i.e., rubble-pixsfm and building-pixsfm \cite{turki2022mega}. 
Experimental results show that the proposed scheme significantly enhances the rendering quality and strictly satisfies the communication resource budget. 
In particular, the proposed STT-GS outperforms the MaxRate \cite{zhang2024efficient} and Fairness \cite{zheng2016wireless} schemes by 4.50\% and 7.81\% in terms of peak signal-to-noise ratio (PSNR). Furthermore, through ablation studies, we also confirm the indispensability of FDC, PTTM, and PAMM algorithms.
To the best of our knowledge, this is the first attempt to integrate GS features and communication constraints into a unified framework.
}

Our main contributions are summarized as follows:
\begin{itemize}
    \item We introduce a sample-then-transmit strategy to support GS-oriented communication designs, resulting in a novel STT-GS framework. By first sampling representative images for loss prediction, STT-GS can prioritize resources towards valuable clients.
    \item We propose the FDC and PTTM techniques to achieve ultra-low pilot costs (e.g., $\leq 10\,\%$ sampling ratio), thereby reserving more transmission resources for the subsequent EGS image collection. 
    \item A cross-layer optimization framework, termed JCSPC, is proposed, which is effectively solved by PAMM. The PAMM-based JCSPC distinguishes heterogeneous view contributions of different clients and mitigates the interference among different clients via power control, thereby improving the image quality compared to existing EGS. 
    \item We implement the proposed STT-GS framework and algorithms on two real-world datasets. 
    Extensive results show significant improvements over existing schemes in terms of diverse metrics. 
    It is found that our method achieves the best tradeoff between view contributions and communication costs.
\end{itemize}

\subsection{Outline and Notations}

\begin{table}[!t]
  \caption{Summary of Important Variables and Parameters}
  \label{TableSum}
  \centering
  \fontsize{8.5pt}{10.5pt}\selectfont 
  \renewcommand{\arraystretch}{0.85} 
  \setlength{\tabcolsep}{3.5pt} 
  \begin{tabular}{|l|l|p{5.5cm}|}
    \hline
    \textbf{Symbol}          & \textbf{Type} & \textbf{Description} \\
    \hline
    $p_{k}$   & Variable      & Transmit power (in $\mathrm{Watt}$) of client $k$. \\
    $x_{k}$   & Variable      & Selection state ($\in\{0,1\}$) of client $k$. \\
    \hline
    $\mathcal{S}$            & Data          & Global GS model. \\
    $\mathcal{G}$            & Data          & 3D Gaussian with trainable parameters. \\
    $\mathcal{E}$            & Data          & Global dataset at the server. \\
    $\mathbf{v}_{i,k}$      & Data          & The $i$-th image of client $k$. \\
    $\mathbf{s}_{i,k}$       & Data          & The $i$-th camera pose of client $k$. \\
    $\mathcal{D}_k$          & Data          & Dataset at client $k \in \mathcal{K}$. \\
    \hline
    $\mathcal{R}(\cdot)$     & Function      & GS inference function. \\
    $\pi_k(\cdot)$           & Function      & Loss over dataset $\mathcal{D}_{k}$ given GS model $\mathcal{S}^{'}$ \\
    $\mathcal{L}(\cdot)$     & Function      & Loss function between image $\widehat{\mathbf{v}}_{i,k}$ and $\mathbf{v}_{i,k}$. \\
    $\text{Train}(\cdot)$    & Function      & GS training function. \\
    $C(\cdot)$               & Function      & Loss function of model $\mathcal{S}^{'}$ on $\bigcup_{k\in\mathcal{X}}\mathcal{D}_{k}$. \\
    $\varphi_1(\cdot)$       & Function      & Penalty function to promote binary solutions. \\
    $\widehat{\varphi}_1(\cdot)$ & Function & Majorized surrogate penalty for $\varphi_1(\cdot)$. \\
    $\varphi_2(\cdot)$       & Function      & Penalty function for non-convex constraints. \\
    $\Phi_k(\cdot)$          & Function      & Original constraint function in problem $P_2$. \\
    $\widehat{\Phi}_k(\cdot)$ & Function     & Majorized surrogate constraint function. \\
    \hline
    $\mathbf{h}_k$ & Parameter & Uplink channel of client $k$. \\
    $H_{k,j}$               & Parameter     & Composite channel gain. \\
    $K$                     & Parameter     & Number of clients, $k \in \{1, \dots, K\}$. \\
    $V_k$                   & Parameter     & The data volume of each sample (in bits). \\
    $B_{k}$                 & Parameter     & Bandwidth allocated to the client $k$ (in $\mathrm{Hz}$). \\
    $\sigma^2$              & Parameter     & Noise power (in $\mathrm{Watt}$). \\
    $R_{k}$                 & Parameter     & Achievable rate (in $\mathrm{bps}$) of the client $k$. \\
    $P_{\mathrm{max}},P_{\mathrm{sum}}$                     & Parameter     & Power budget of the client. \\
    $T$                     & Parameter     & Transmission time threshold. \\
    \hline
  \end{tabular}
\end{table}

The remainder of this paper is organized as follows.
Section \ref{section2} states the EGS problem. 
Section \ref{section3} describes the STT-GS system architecture.
Subsequently, Sections~\ref{section4} and \ref{section5} present the loss prediction and full transmission algorithms. 
Section~\ref{section6} presents experimental results. 
Finally, Section~\ref{section7} concludes this work. 

\emph{Notation}:
Italic, bold, capital bold, and curlicue letters represent scalars, vectors, matrices, and sets, respectively. $\nabla f$ represents the gradient of a function $f$. $\mathbb{E}(x)$ represents the expectation of a random variable $x$. 
$\|\cdot\|$ denotes the vector norm function. $\lceil \cdot \rceil$ is the ceiling function. 
$|\cdot|$ is the magnitude of a scalar or the cardinality of a set, and $[\,\cdot\,]^T$ denotes the matrix transpose operation.
The operators $(\cdot)^{T}$, $(\cdot)^{H}$, and $(\cdot)^{-1}$ take the transpose, Hermitian, and inverse of a matrix, respectively.
The symbol $\mathbf{I}_{N}$ indicates the $N\times N$ identity matrix,
$\mathbf{1}_{N}$ indicates the $N\times 1$ vector with all entries being unity, $\mathrm{diag}(\mathbf{x})$ indicates the diagonal matrix with diagonal elements being vector $\mathbf{x}$, and $\mathcal{CN}(0,1)$ stands for complex Gaussian distribution with zero mean and unit variance.
Symbols $\in\mathbb{R}_+$ and $\in\mathbb{C}$ denote non-negative real numbers and complex numbers, respectively. 
Important variables and parameters are summarized in Table~\ref{TableSum}.

\section{Edge Gaussian Splatting} \label{section2}

We consider an EGS system as shown in Fig.~\ref{fig1}, which consists of an edge server with $N$ antennas and $K$ single-antenna clients. 
The objective of the EGS system is to reconstruct a 3D scene by aggregating distributed data from clients $\mathcal{K}=\{1,\cdots,K\}$ and training a global GS model $\mathcal{S}=\{\mathcal{G}_1,\mathcal{G}_2,\cdots\}$ using edge GPUs, where each $\mathcal{G}_i$ denotes a 3D Gaussian with trainable parameters.
Below, we present our system model and problem formulation in detail.

\subsection{System Model}

\begin{figure*}[!t]
\centering 
\includegraphics[width=0.98\linewidth]{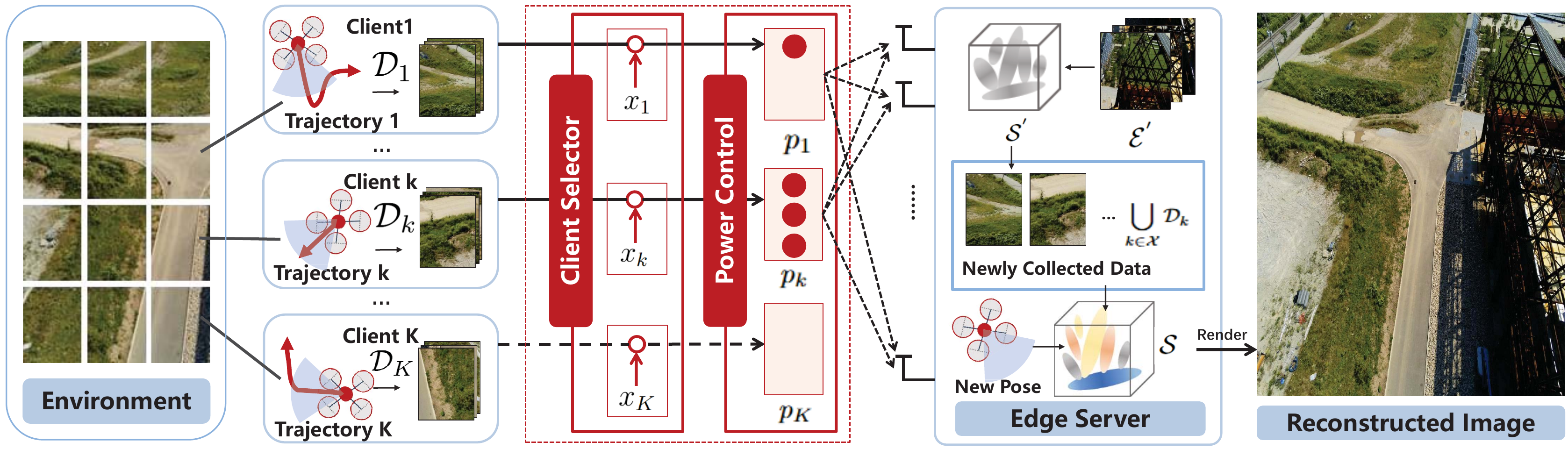}
\caption{The EGS system with client selection, power allocation, and beamforming receiver.}
\label{fig1}
\end{figure*}

Specifically, the dataset at client $k \in \mathcal{K}$ is given by
\begin{align}
\mathcal{D}_k = \left\{\mathbf{v}_{i,k}, \mathbf{s}_{i,k}\right\}_{i=1}^{l_k},
\end{align}
where $\mathbf{v}_{i,k}\in \mathbb{R}^{3LW}$ denotes the $i$-th image of client $k$, where $L$ and $W$ are the length and width of camera images respectively, and the coefficient $3$ accounts for the red green blue (RGB) channels. 
The $\mathbf{s}_{i,k}\in \mathbb{R}^{6}$ represents the associated 6D camera pose \cite{kerbl20233d}, and $l_k=|\mathcal{D}_k|$ is the number of samples at client $k$. 
The data volume in bits of each sample is $V_k$.
These data can be locally generated at each client by adopting monocular cameras and localization packages along the trajectory $(\mathbf{s}_{1,k},\mathbf{s}_{2,k},\cdots,\mathbf{s}_{l_k,k})$, which are standard modules equipped at modern robot platforms \cite{xu2022fast} (e.g., drones).

To avoid multi-user interference during dataset uploading, a client selection module with a binary vector $\mathbf{x}=[x_1,\cdots, x_K]^T$ and $x_k\in\{0,1\}$ is introduced, where $x_k=1$ indicates that the $k$-th client is being selected and $x_k=0$ otherwise. 
For a selected client $k$, e.g., $x_k=1$, it transmits its signal $z_{k}$ with power $\mathbb{E}[|z_{k}|^2]=p_{k}$ for uploading $\mathcal{D}_k$.
Accordingly, the aggregated signal \cite{wang2020learning} received at the edge server is written as
\begin{equation}
    \mathbf{r} = \sum_{k=1}^K x_k\mathbf{h}_{k} z_{k}+\mathbf{n},
    \label{eq:signal}
\end{equation}
where $\mathbf{r}=[r_1,\cdots,r_N]^T\in\mathbb{C}^{N\times 1}$ is the received signal, $\mathbf{h}_{k} \in \mathbb{C}^{N \times 1}$ denotes the channel from the $k$-th client to the edge server, and $\mathbf{n}\in\mathbb{C}^{N\times 1}$ is additive white Gaussian noise (AWGN) with zero mean and covariance $\mathbb{E}\{\mathbf{n}\mathbf{n}^{H}\}=\sigma^2\mathbf{I}_{N}$,
with $\mathbf{I}_{N}$ being the identity matrix of size $N\times N$. 
The wireless channel is assumed to follow Rician fading\cite{wang2020angle}, i.e., 
\begin{equation} 
\mathbf{h}_k = \sqrt{h_0 \omega_k d_k^{-\alpha}} \left( \sqrt{\frac{K_{\text{Ric}}}{K_{\text{Ric}} + 1}} \mathbf{h}^{\text{LOS}}_{k} + \sqrt{\frac{1}{K_{\text{Ric}} + 1}} \mathbf{h}^{\text{NLOS}}_{k} \right),
\nonumber
\end{equation}
where $h_0 = -30 \, \text{dB}$ is the path loss at $1 \, \text{m}$, $\omega_k$ is the shadow fading of client $k$, $d_k$ is the distance between client $k$ and the server, $\alpha$ is the path loss exponent, 
and $K_{\text{Ric}}$ is the Rician K-factor. 
Note that $\{d_k\}$ depends on clients' locations, and varies for different clients.

The LoS component $\mathbf{h}_{k}^{\mathrm{LOS}}$ is
\begin{align}\label{LOS}
\mathbf{h}_{k}^{\mathrm{LOS}}&=\Big[1,\mathrm{exp}\left(-\mathrm{j}\pi\,\mathrm{sin}\,\theta_k\right)
,\cdots,
\nonumber\\
&\quad{}
\mathrm{exp}\left(-(N-1)\,\mathrm{j}\pi\,\mathrm{sin}\,\theta_k\right)\Big]^T,
\end{align}
where $\theta_k\in\mathcal{U}(-\pi,+\pi)$.
The non-LoS component $\mathbf{h}_{k}^{\mathrm{NLOS}}\sim\mathcal{CN}(\mathbf{0},\mathbf{I}_N)$.

To separate the useful components $\mathbf{z} = [z_1,...,z_k]$ from the received signal $\mathbf{r}$, it is necessary to mitigate the interference and noise. An ideal approach is to apply a minimum mean square error (MMSE) combing vector to process the signal $\mathbf{r}$ \cite{lu2014overview}.
However, MMSE involves calculating the matrix inverse, which is a computationally intensive operation especially for EGS with a large number of antennas. 
Therefore, we choose the celebrated maximum ratio combining (MRC) as in \cite{song2025distributed,wang2020learning} with $\mathbf{w}_k=
\left\Vert\mathbf{h}_{k}\right\Vert_2^{-1}
\mathbf{h}_{k}$ due to its high computational efficiency. It has been proved in \cite{song2025distributed} that MRC achieves an asymptotically equivalent performance compared to MMSE (i.e., the benefit brought by MMSE over MRC vanishes as $N$ increases).

Based on the above analysis and applying $\mathbf{w}_k=
\left\Vert\mathbf{h}_{k}\right\Vert_2^{-1}
\mathbf{h}_{k}$ to $\mathbf{r}$, the uplink data rate of client $k$ is given by
\begin{align}
&R_{k}=B_k\mathrm{log}_2\left(1+\frac{H_{k,k}p_{k}}{\sum_{j=1,j\neq k}^KH_{k,j}p_{j}+
\sigma^2} \right), \label{Rk}
\end{align}
where $B_k$ in Hz is the bandwidth of the client $k$, and $H_{k,j}$ represents the composite channel gain (including channel fading and antenna processing) from client $j$ to the edge when decoding data from client $k$:
\begin{align}
&H_{k,j}=
\left\{
\begin{aligned}
&\left\Vert\mathbf{h}_{k}\right\Vert_2^2
,&\mathrm{if}~k=j
\\
&\frac{|\mathbf{h}_k^H\mathbf{h}_{j}|^2}{\left\Vert\mathbf{h}_{k}\right\Vert_2^2}
,&\mathrm{if}~k\neq j
\end{aligned}
\right.
.
\end{align}
It can be seen that different clients have different interference conditions, which complicates the $R_k$'s expression. 

With the newly collected data from clients $\mathcal{X}=\{k\in\mathcal{K}:x_k=1\}$, the server updates its global dataset as 
\begin{equation}
    \mathcal{E} = \left(\bigcup_{k\in\mathcal{X}}\mathcal{D}_{k} \right) \bigcup \mathcal{E}^{'},
    \label{eq:update_central_dataset_with_client_selection}
\end{equation}
where $\mathcal{E}^{'}$ denotes the historical dataset at the server.
Then, the server trains the GS model based on $\mathcal{E}$ as
\begin{equation}
    \mathcal{S} = \text{Train}\left(\mathcal{E}|\mathcal{S}^{'}\right),
    \label{eq:train_global_model}
\end{equation}
where $\text{Train}(\cdot)$ represents the training procedure for the GS model on dataset $\mathcal{E}$, given the initial GS model $\mathcal{S}^{'}$ obtained from previous trainings.

\begin{figure}[!t]
    \centering
    \includegraphics[width=0.98\columnwidth]{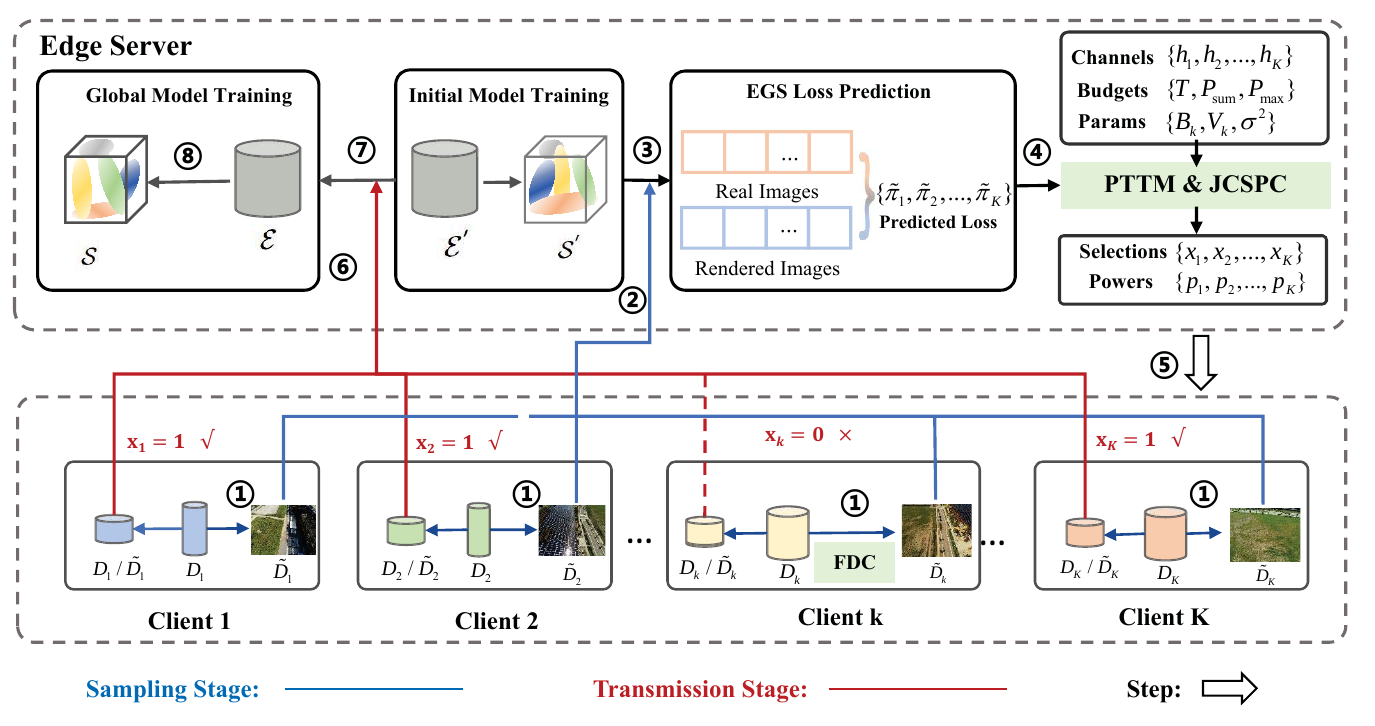}
    \caption{Workflow of STT-GS: (1) FDC-based sampling at clients, (2) transmission of sampled data, (3) loss prediction, (4) PTTM \& JCSPC optimization, (5) client selection and power allocation, (6) full data transmission, (7) server-side data aggregation, and (8) global model training. Communication is divided into a sampling stage (blue) and a transmission stage (red).}
    \label{fig:overview} 
\end{figure}

\subsection{Problem Formulation}

In the considered system, the design variables that can be controlled are the client selection vector $\mathbf{x}=[x_1,\cdots,x_K]^T$ and transmit power vector $\mathbf{p}=[p_1,\cdots,p_K]^T$.
To ensure low power consumption
and limited inter-user interference, both the individual and total powers must satisfy their prescribed limits, given by $P_{\text{max}}$ and $P_{\mathrm{sum}}$, respectively. Specifically, each client's transmit power must fulfill the constraint $\{0 \leq p_k \leq P_{\text{max}}, \forall k\}$ and $\sum_{k=1}^Kx_kp_k\leq P_{\mathrm{sum}}$ \cite{yu2004iterative, zhang2024efficient, al2011achieving, beck20141}. 
On the other hand, the dataset uploading at each selected client must be completed within a time threshold $T$, i.e., $x_kV_k|\mathcal{D}_k|R_k^{-1}\leq T$.\footnote{
One may wonder if the data uploading time $T$ is negligible compared to the GS model training time.
This is in fact, not true, due to the recent advancements in 3D GS training algorithms \cite{mallick2024taming,chen2025dashgaussian} (e.g., DashGaussian \cite{chen2025dashgaussian}), which have dramatically reduced model training time to merely $3.23$ minutes. 
Given this significant time reduction in model training, introducing a time threshold $T$ for data transmission becomes essential for EGS.}

Having the power and time constraints satisfied, it is then crucial to maximize the information gain contributed by $\mathcal{E}\setminus\mathcal{E}'= \bigcup_{k\in\mathcal{X}}\mathcal{D}_{k}$ for better 3D reconstruction. 
In other words, we need to collect those datasets that could change the GS model from $\mathcal{S}^{'}$ to $\mathcal{S}$ with the maximum extent. 
According to the uncertainty sampling theory~\cite{yoo2019learning, zhu2009active, beluch2018power}, this is equivalent to maximizing the total inference loss of model $\mathcal{S}^{'}$ on $\bigcup_{k\in\mathcal{X}}\mathcal{D}_{k}$, denoted as $C(\mathcal{S}^{'},\{\mathcal{D}_{k}\},\mathbf{x})$.
Specifically, the GS inference function $\mathcal{R}(\cdot|\mathcal{S}^{'})$ is able to render a photo-realistic image from a camera pose as $\widehat{\mathbf{v}}_{i,k}=\mathcal{R}(\mathbf{s}_{i,k}|\mathcal{S}^{'})$.
To derive an explicit form of $C$, we need to measure the difference between the rendered image $\widehat{\mathbf{v}}_{i,k}$ (produced using $\mathcal{S}^{'}$) and ground-truth image $\mathbf{v}_{i,k}$. 
According to \cite{kerbl20233d}, a proper metric is the vanilla GS that adopts the following loss function: 
 \begin{align}\label{LGS}
\mathcal{L}\left(\widehat{\mathbf{v}}_{i,k},\mathbf{v}_{i,k}\right)
&=(1-\lambda)\mathcal{L}_{1}(\widehat{\mathbf{v}}_{i,k},\mathbf{v}_{i,k}) 
+
  \lambda \mathcal{L}_{\mathsf{DSSIM}}(\widehat{\mathbf{v}}_{i,k},\mathbf{v}_{i,k}),
\end{align}
where functions $\mathcal{L}_{1},\mathcal{L}_{\mathsf{DSSIM}}$ are given by 
\begin{align}
\mathcal{L}_{1}(\widehat{\mathbf{v}}_{i,k},\mathbf{v}_{i,k}) &=\|\widehat{\mathbf{v}}_{i,k}-\mathbf{v}_{i,k}\|_1,
\\
\mathcal{L}_{\mathsf{DSSIM}}(\widehat{\mathbf{v}}_{i,k},\mathbf{v}_{i,k})
&=
{1-\mathcal{L}_{\mathsf{SSIM}} (\widehat{\mathbf{v}}_{i,k},\mathbf{v}_{i,k})},
\end{align}
respectively, with $\mathcal{L}_{\mathsf{SSIM}}$ being the SSIM function detailed in \cite[Eqn. 5]{wang2011ssim}, 
and the weight $\lambda\in[0,1]$ is set to $\lambda=0.2$ according to \cite{kerbl20233d}.

Based on the above analysis and by aggregating the losses over all samples and clients, 
the total loss function $C$ is
\begin{align}
&C\left(\mathcal{S}^{'},\{\mathcal{D}_{k}\},\mathbf{x}\right)
=\sum_{k=1}^{K} x_k \, 
\underbrace{\sum_{(\mathbf{v}_{i,k},\mathbf{s}_{i,k})\in \mathcal{D}_{k}} \mathcal{L}\left(
\mathcal{R}(\mathbf{s}_{i,k}|\mathcal{S}^{'})
,\mathbf{v}_{i,k}
\right)}_{\pi_k(\mathcal{S}^{'},\mathcal{D}_{k})},
    \label{loss_func}
\end{align}
where we define $\pi_k(\mathcal{S}^{'},\mathcal{D}_{k})$ as the loss function over dataset $\mathcal{D}_{k}$ given GS model $\mathcal{S}^{'}$.
This leads to the following EGS optimization problem:
\begin{subequations}
  \label{eq:OP_and_constraints}
  \begin{align}
    \mathsf{P}: \max_{\substack{\mathbf{x},\mathbf{p}}} 
      \quad & C\left(\mathcal{S}^{'},\{\mathcal{D}_{k}\},\mathbf{x}\right) 
      \label{eq:op_with_prediction_loss} \\
    \text{s.t.}\quad & 
      TB_k\log_2 \left( 1 + \frac{x_k p_{k}H_{k,k}}
	{
    \sum_{j\neq k} x_jp_jH_{k,j}+\sigma^2} \right)
    \nonumber\\
    & \geq x_kV_k|\mathcal{D}_k|, ~ \forall k,
    \label{eq:load1_constraints} \\ 
      & 0 \leq p_k \leq P_{\text{max}}, ~ \forall k, ~ \sum_{k=1}^K p_k  \leq P_{\text{sum}},
      \label{eq:power_limits} \\
    & x_k \in \{0, 1\}, ~ \forall k.
      \label{eq:binary_constraint} 
  \end{align}
\end{subequations}

The major obstacle to solving $\mathsf{P}$ is the inaccessibility of the cost function $C$. 
This is because at the optimization stage, the edge server has not yet received the raw image data $\{\mathbf{v}_{i,k}\}$ from the clients, and the loss $\mathcal{L}$ in \eqref{LGS} cannot be evaluated directly.
To address the challenge, Section~\ref{section3} will propose a two-stage STT-GS scheme that can efficiently estimate $\mathcal{L}$ and predict the rendering loss prior to full data transmissions, thereby enabling efficient resource allocation for EGS.

\begin{remark}
We assume that the uplink channels $\{H_{k,j}\}$ in problem P are known at the edge server by using the pilot signals. If the estimated values of  $\{H_{k,j}\}$ involve errors, we can reformulate P as robust EGS by constraining the outage probability (OP) of each client below a certain threshold. Such OP constraints can be handled by Bernstein-Type Inequalities.
\end{remark}

\section{Proposed Sample-Then-Transmit Strategy} \label{section3}

The architecture of our STT-GS scheme is shown in Fig. \ref{fig:overview}. 
The input consists of the distributed datasets $\{\mathcal{D}_k\}$ at clients, the historical dataset $\mathcal{E}^{'}$ at the server which is used to pretrain the initial GS model $\mathcal{S}'$, channels $\{\mathbf{h}_k\}$, time budget $T$, power budgets $(P_{\text{max}},P_{\text{sum}})$, and hyper-parameters $(B_k,V_k,\sigma^2)$.
The outputs of the system consist of the client selections $\mathbf{x}$, transmit powers $\mathbf{p}$, and the trained GS model $\mathcal{S}$.
The goal of sample-then-transmit EGS is to collect the most valuable dataset $\mathcal{E}$ from the proper clients, under strict time and power budgets $(T,P_{\text{max}},P_{\text{sum}})$, so as to improve the rendering quality of $\mathcal{S}$. 

The entire pipeline is divided into the following steps:
\begin{itemize}
    \item[1)] Sampling: Each client samples a sub-dataset $\{\widetilde{\mathcal{D}}_k \subseteq \mathcal{D}_k, \forall k\}$ according to the FDC module (will be detailed in Section IV-A);    
    \item[2)] First-stage transmission: Clients upload sub-datasets $\{\widetilde{\mathcal{D}}_k, \forall k\}$ (i.e., pilot data) to the server using PTTM (will be detailed in Section IV-B);
    \item[3)] Validation: The edge server renders images $\{\widehat{\mathbf{v}}_{i,k}\}$ based on $\mathcal{S}'$ and $\{\widetilde{\mathcal{D}}_k, \forall k\}$; 
    \item[4)] Prediction: The edge server estimates the GS loss as $\pi(\mathcal{S}^{'},\mathcal{D}_k)
    \approx \tilde{\pi}(\mathcal{S}^{'},\tilde{\mathcal{D}}_k)$ using $\{\widetilde{\mathcal{D}}_k\}$;
    \item[5)] Scheduling: The edge server selects clients (i.e., $\mathbf{x}$) and controls powers $\mathbf{p}$ by solving $\mathrm{P}$ with $C(\mathcal{S}^{'},\{\mathcal{D}_{k}\},\mathbf{x})\approx \sum_{k=1}^Kx_k\tilde{\pi}(\mathcal{S}^{'},\tilde{\mathcal{D}}_k) 
    $ using PAMM (will be detailed in Section V);
    \item[6)] Second-stage transmission: Selected clients upload their remaining datasets $\{\mathcal{D}_k\setminus\widetilde{\mathcal{D}}_k, \forall k~\mathrm{with}~x_k=1\}$ according to the scheduling results; 
    \item[7)] Data aggregation: The server aggregates datasets
    $\mathcal{E}=\mathcal{E}^{'}\cup\{\widetilde{\mathcal{D}}_k, \forall k\}\cup \{\mathcal{D}_k\setminus\widetilde{\mathcal{D}}_k, \forall k~\mathrm{with}~x_k=1\}$; 
    \item[8)] GS training: Train a new GS model $\mathcal{S}$ based on $\mathcal{E}$.
\end{itemize}

The above $8$ steps can be categorized into two main stages, i.e., EGS loss prediction and EGS full transmission, where each stage consists of $4$ steps. 
\begin{itemize}
\item EGS loss prediction consists of steps 1--4, which determines how to sample the data, upload the sampled data, validate the data, and predict the loss. This first-stage transmission predicts the loss function $C$ by uploading subsets $\{\widetilde{\mathcal{D}}_k\}_{k=1}^K$.
\item EGS full transmission consists of steps 5--8, which involves client scheduling, full data uploading, data aggregation, and GS model training. This second-stage transmission completes the data collection of $\{\mathcal{D}_k, \forall k~\mathrm{with}~x_k=1\}$.
\end{itemize} 
The following sections present the two stages in detail. 

\section{EGS Loss Prediction}\label{section4}

In the EGS loss prediction stage, each client selects a small but representative subset  $\tilde{\mathcal{D}}_k \subseteq \mathcal{D}_k$ known as pilot data, defined as:
\begin{equation}
  \lvert \tilde{\mathcal{D}}_k \rvert = \lceil \rho_k \lvert \mathcal{D}_k \rvert \rceil,\quad 0 < \rho_k < 1,
  \label{eq:sampling_process}
\end{equation}
where $\lceil \cdot \rceil$ is the ceiling function and $\rho_k$ is the sampling ratio. 
A larger $\rho_k$ leads to better GS loss prediction, but incurs higher transmission overhead; and vice versa. 
In practice, we can fine-tune $\rho_k$ within a certain interval (e.g., $5\%\sim 20\%$) exploiting cross-validation to achieve a desirable balance between prediction accuracy and transmission overhead. For instance, we observe in our experiments that $\rho_k=10\%$ achieves excellent trade-off (will be detailed in Section VI).

Given $\rho_k$, the next question is how to determine $\tilde{\mathcal{D}}_k$. 
A naive approach is to adopt random sampling, which randomly chooses an image inside $\mathcal{D}_k$ each time until $\lvert \tilde{\mathcal{D}}_k \rvert$ images are collected \cite{mayer2020adversarial}.
Another approach is to adopt uniform sampling, which orders the images in $\mathcal{D}_k$ according to the timestamps and selects an image for transmission for every $\lvert \mathcal{D}_k \rvert/\lvert \tilde{\mathcal{D}}_k \rvert$ image frames. 
These methods do not take the importance of images into account, and may lead to redundancy in transmitting similar images. 
For example, if the robot remains stationary, multiple images may be nearly identical, uploading them yields little additional value. To address the problem, the following subsection will present a method that improves sampling efficiency compared to the above approaches.

\subsection{Feature Domain Clustering for Efficient Sampling}

We propose the FDC method that selects the most representative images inside $\mathcal{D}_k$, such that the information loss between $\tilde{\mathcal{D}}_k$ and $\mathcal{D}_k$ is minimized. 
The goal is equivalent to maximizing the similarity between the underlying distributions of $\tilde{\mathcal{D}}_k$ and $\mathcal{D}_k$.
Motivated by such insights, the FDC scheme first adopts a feature extraction module to convert raw images into their embedding feature spaces, and then clusters the features into $\lvert \tilde{\mathcal{D}}_k \rvert$ groups with nonuniform sampling.
For each cluster, we select the representative image by minimizing the $l_2$ norm between its feature and the centroid. 
By iterating over all clusters, we obtain the sampled dataset $\tilde{\mathcal{D}}_k$. 
The entire workflow of FDC is shown in Fig.~\ref{fig:GS-data-sampling}. 

\begin{figure}[!t]
    \centering
    \includegraphics[width=0.98\columnwidth]{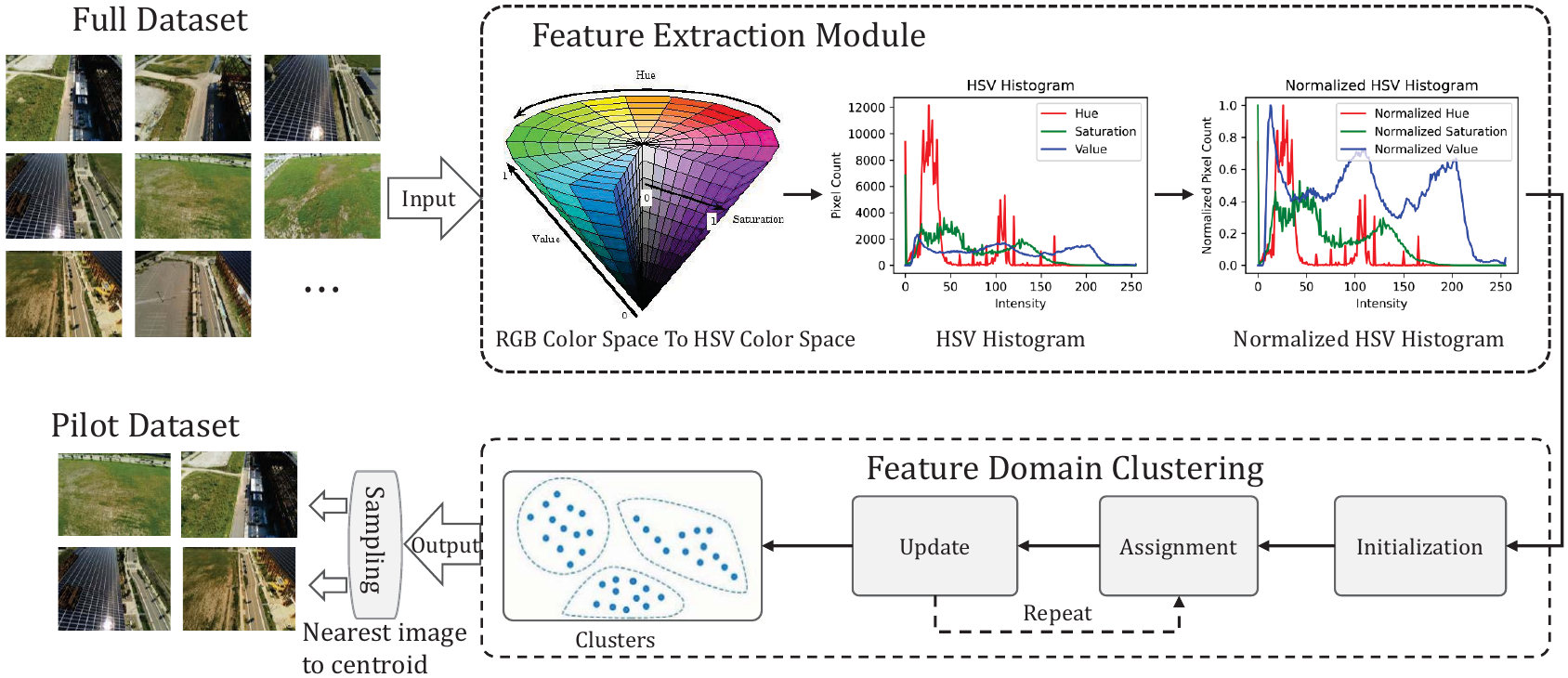}
    \caption{Sampling workflow of FDC.}
    \label{fig:GS-data-sampling} 
\end{figure}

Specifically, we consider the Hue Saturation Value (HSV) feature representation \cite{kuehni2008color}, where 
H is hue angle, S is saturation intensity, and V is value brightness. 
The feature extraction process is thus
\begin{equation}
    \mathbf{v}_{i,k}^{\text{HSV}} = E_{\text{HSV}}\left(\mathbf{v}_{i,k}\right),
    \label{HSV}
\end{equation}
where function $E_{\text{HSV}}$ converts an RGB image $\mathbf{v}_{i,k}$ into an HSV image $\mathbf{v}_{i,k}^{\text{HSV}}$ (flattened into a vector). 
Note that all elements of $\mathbf{v}_{i,k}^{\text{HSV}}$ are normalized between $[0,1]$.

With $\{\mathbf{v}_{i,k}^{\text{HSV}}\},\forall i,k$, we then solve the following optimization problem:
\begin{align}
   \mathsf{Q}_1: & \min_{\{\mathcal{C}_{m,k},\mathbf{c}_{m,k}\}} 
     \sum_{m=1}^{|\tilde{\mathcal{D}}_k|} \sum_{\mathbf{y}_{i,k} \in \mathcal{C}_{m,k}} \|\mathbf{y}_{i,k} - \mathbf{c}_{m,k}\|^2\ 
    \nonumber\\
     & \ \ \mathrm{s.t.} \ \ 
    \mathcal{C}_{1,k}\bigcup\mathcal{C}_{2,k}\cdots \bigcup\mathcal{C}_{|\tilde{\mathcal{D}}_k|,k} =\{\mathbf{v}_{i,k}^{\text{HSV}}\}_{i=1}^{|\mathcal{D}_k|},
\end{align}
where $\mathcal{C}_{m,k} $ is the set of feature vectors in the $ m $-th cluster, and $\mathbf{c}_{m,k}$ is the centroid of cluster $\mathcal{C}_{m,k}$. 
Problem $\mathsf{Q}_1$ can be addressed by alternating minimization (AM) that iterates between solving $\mathcal{C}_{m,k} $'s subproblem and $\mathbf{c}_{m,k} $'s subproblem. 
The AM method will converge to a local optimal solution to $\mathsf{Q}_1$. 
The local optimal solution of centroids and clusters provided by AM are denoted as $\{\mathbf{c}_{m,k}^*\}$ and $\{\mathcal{C}_{m,k}^*\}$, respectively.

Finally, given $\{\mathbf{c}_{m,k}^*\}$ and $\{\mathcal{C}_{m,k}^*\}$, 
the representative image of the $m$-th cluster at client $k$ is selected as the one closest to associated centroid:
\begin{equation}
    \tilde{\mathbf{v}}_{m,k} = \mathrm{arg}\min_{ \mathbf{y}_{i,k} \in\mathcal{C}_{m,k}^* } \|\mathbf{y}_{i,k}-\mathbf{c}_{m,k}^*\|^2.
    \label{HSV}
\end{equation}
The sampled dataset at client $k$ is 
$\tilde{\mathcal{D}}_k=\{\tilde{\mathbf{v}}_{m,k},\tilde{\mathbf{s}}_{m,k}\}_{m=1}^{|\tilde{\mathcal{D}}_k|}$, where $\tilde{\mathbf{s}}_{m,k}$ is the pose of camera image $\tilde{\mathbf{v}}_{m,k}$.
Finally, all $\tilde{\mathcal{D}}_k$ will be uploaded to the server during the first-stage transmission.
The entire procedure of the FDC method is summarized in Algorithm~\ref{alg:GS_Data_Clustering}.

\begin{algorithm}[!t]
\SetAlgoLined
\DontPrintSemicolon
\KwIn{
Clients' datasets $\mathcal{D}=\{\mathcal{D}_1,\ldots,\mathcal{D}_K\}$ and sampling ratio $\rho_k$.}
\KwOut{Pilot datasets $\tilde{\mathcal{D}}=\cup_{k=1}^{K} \tilde{\mathcal{D}}_k$.}

\For{each client $k$ in $1$ to $K$}{
    $\lvert \tilde{\mathcal{D}}_k \rvert = \lceil \rho_k \lvert \mathcal{D}_k \rvert \rceil$\;
    
    \For{$\mathbf{v}_{i,k}$ in $\mathcal{D}_k$}{
        $\mathbf{v}_{i,k}^{\text{HSV}} = E_{\text{HSV}}(\mathbf{v}_{i,k})$\;
    }
    Solve $\mathrm{Q}_1$ using AM and obtain $\{\mathbf{c}_{m,k}^*, \mathcal{C}_{m,k}^*\}$

    \For{$m$ in $1$ to $\lvert \tilde{\mathcal{D}}_k \rvert$}{
        $\tilde{\mathbf{v}}_{m,k} = \arg\min_{\mathbf{y}_{i,k} \in \mathcal{C}_{m,k}^*} \|\mathbf{y}_{i,k}-\mathbf{c}_{m,k}^*\|$\;
        $\tilde{\mathbf{s}}_{m,k} = \text{state of } \tilde{\mathbf{v}}_{m,k}$\;
    }

    $\tilde{\mathcal{D}}_k=\{\tilde{\mathbf{v}}_{m,k},\tilde{\mathbf{s}}_{m,k}\}_{m=1}^{|\tilde{\mathcal{D}}_k|}$\;
}

$\tilde{\mathcal{D}}=\cup_{k=1}^{K} \tilde{\mathcal{D}}_k$\;
\Return $\tilde{\mathcal{D}}$\;

\caption{FDC for EGS Loss Prediction}
\label{alg:GS_Data_Clustering}
\end{algorithm}

\subsection{Pilot Transmission Time Minimization}

Having determined the sampling strategy, it is then crucial to design the associated transmission strategy for uploading the sampled pilot data.
Since at this point, we have no access to the clients' data, the best approach is to upload the pilot data as quickly as possible to reserve more transmission time for the second stage.
This leads to the pilot transmission time minimization (PTTM) problem of the first stage:
\begin{subequations}
  \label{eq:OP_sampling}
  \begin{align}
    \mathsf{P}_1: \min_{\substack{T_0,\mathbf{p}}}
      \quad & T_0  \\
   \ \text{s.t.} \ \ \ \
   &
      \log_2 \left( 1 + \frac{p_{k}H_{k,k}}
	{
    \sum_{j\neq k} p_jH_{k,j}+\sigma^2} \right) \geq \frac{V_k |\tilde{\mathcal{D}_{k}}|}{T_0B_k}, ~ \forall k,  \label{20b}  \\ 
      & \textsf{constraints } (\ref{eq:power_limits}),
  \end{align}
\end{subequations}
where $T_0$ denotes the time of uploading pilot data.
By taking the exponential on both sides of constraint \eqref{20b}, $\mathsf{P}_1$ can be reformulated into a quasi-linear optimization problem, which can be solved optimally by the iterative bisection search and commercial convex optimization techniques.

Now, we discuss the algorithm for solving problem $\mathsf{P}_1$ in detail. Specifically, given an upper bound for the bisection searching interval, say $T_{\mathrm{max}}$, and a lower bound $T_{\mathrm{min}}$, the trial point is set to $\kappa=(T_{\mathrm{max}}+T_{\mathrm{min}})/2$.
If $\mathsf{P}_1$ with $T_0=\kappa$ is feasible, the upper bound is updated as $T_{\mathrm{max}}=\kappa$; otherwise, the lower bound is updated as $T_{\mathrm{min}}=\kappa$.
The process is repeated until $|T_{\mathrm{max}}-\kappa|<\epsilon$, where $\epsilon$ is a small positive constant to control the accuracy.

As $T_0\geq0$, an initial $T_{\mathrm{min}}$ can be chosen as $0$.
On the other hand, according to the time budget, a valid initial $T_{\mathrm{max}}$ can be set to $T_{\mathrm{max}}=T$.
With the obtained bounds, we are now ready to apply the bisection algorithm to find $T^*_0$ in $\mathsf{P}_1$.
An efficient way to provide a feasibility check of $\mathsf{P}_1$ given $T_0=\kappa$ is to first minimize the total transmit power via the following linear programming problem:
\begin{subequations}
  \label{eq:OP_sampling}
\begin{align}
    \mathsf{F}_1: \min_{\substack{\mathbf{p}}}
      \quad & 
      \sum_{k=1}^K p_k
      \\
   \ \text{s.t.} \ \ \ \
   &
\left(2^{\frac{V_k |\tilde{\mathcal{D}_{k}}|}{T_0B_k}}-1\right)
    \left(\sum_{j\neq k} p_jH_{k,j}+\sigma^2\right)
    \nonumber\\
    &
    -p_{k}H_{k,k}
    \leq 0
    , ~ \forall k,  \label{21b}  \\ 
      &0 \leq p_k \leq P_{\text{max}}, ~ \forall k.
\end{align}
\end{subequations}
and then check whether the optimal $\{p_k^*\}$ satisfies $ \sum_{k=1}^K p_k\leq P_{\mathrm{sum}}$.
If so, problem $\mathsf{P}_1$ with $T_0=\kappa$ is feasible;
otherwise, the transmit power at clients cannot support time $\kappa$ and $\mathsf{P}_1$ with $T_0=\kappa$ is infeasible.
In conclusion, the procedure for computing the first-stage power allocation is summarized in Algorithm 2.

\begin{algorithm}[!t]
    \caption{PTTM}
        \begin{algorithmic}[1]
\State \textbf{Input} $\{H_{k,j}\}$, $\sigma^2$, $T$, $\tilde{\mathcal{D}}_k$, $P_{\text{sum}}$, $P_{\text{max}}$, $B_k$, $V_k$.
            \State \textbf{Output} $T_0,\mathbf{p}$.
            \State \textbf{Repeat}
            \State \ \ \ Update $\kappa=(T_{\mathrm{min}}+T_{\mathrm{max}})/2$.
            \State \ \ \ Solve problem $\mathsf{F}_1$.
            \State \ \ \ Update $T_{\mathrm{max}}=\kappa$ if $\mathsf{F}_1$ is feasible.
             \State \ \ \ Update $T_{\mathrm{min}}=\kappa$ if $\mathsf{F}_1$ is infeasible.
            \State \textbf{Until} $|T_{\mathrm{max}}-\kappa|<\epsilon$, where $\epsilon$ is a small positive constant to control the accuracy.
        \end{algorithmic}
\end{algorithm}

\subsection{Pilot Validation and Loss Estimation}

With the received images $\{\tilde{\mathcal{D}}_k\}$, the edge server can predict the rendering loss for each client and selects the most informative clients for further data transmission. 
In particular, the server renders the GS images 
$\mathcal{R}(\tilde{\mathbf{s}}_{m,k}|\mathcal{S}^{'})$ 
from poses $\tilde{\mathbf{s}}_{m,k}$. 
Then, the GS loss for all samples in $\{\tilde{\mathcal{D}}_k\}$ is given by 
\begin{align}
\psi_k(\mathcal{S}^{'},\tilde{\mathcal{D}}_k)
=
\sum_{(\tilde{\mathbf{v}}_{i,k},\tilde{\mathbf{s}}_{i,k})\in \tilde{\mathcal{D}_{k}}} \mathcal{L}\left(
\mathcal{R}(\tilde{\mathbf{s}}_{m,k}|\mathcal{S}^{'})
,\tilde{\mathbf{v}}_{i,k}
\right).
    \label{eq:avg_loss}
\end{align}
Therefore, the actual loss $\pi_k(\mathcal{S}^{'},\mathcal{D}_k)$ in problem $\mathrm{P}$, can be safely approximated by $\tilde{\pi}_k(\mathcal{S}^{'},\tilde{\mathcal{D}}_k)$ as 
\begin{align}\label{appgs}
    \pi_k(\mathcal{S}^{'},\mathcal{D}_k) \approx 
    \tilde{\pi}_k(\mathcal{S}^{'},\tilde{\mathcal{D}}_k)
    =
    \frac{|\mathcal{D}_{k}|}{|\tilde{\mathcal{D}_{k}}|}
    \psi_k(\mathcal{S}^{'},\tilde{\mathcal{D}}_k).
\end{align}
Accordingly, the loss function in $\mathsf{P}$ is approximated as 
\begin{align}
&C\left(\mathcal{S}^{'},\{\mathcal{D}_{k}\},\mathbf{x}\right)
\approx
\sum_{k=1}^{K} x_k \, 
\frac{|\mathcal{D}_{k}|}{|\tilde{\mathcal{D}_{k}}|}
    \psi_k(\mathcal{S}^{'},\tilde{\mathcal{D}}_k).
    \label{loss_func}
\end{align}

\section{EGS Full Transmission With\\Joint Client Scheduling and Power Allocation}\label{section5}

In the EGS full transmission stage, the key is to select the most valuable client for GS training based on the predicted GS loss in Section IV.
The problem is formulated as 
\begin{subequations}
  \begin{align}
    &\mathsf{P}_2: \max_{\substack{\mathbf{x},\mathbf{p}}} 
      \quad  \sum_{k=1}^{K} x_k \, 
    \tilde{\pi}_k(\mathcal{S}^{'},\tilde{\mathcal{D}}_k), 
      \label{eq:op_with_avg_loss} \\
    & \ \text{s.t.} \
      \log_2 \left( 1 + \frac{x_k p_{k}H_{k,k}}
	{
    \sum_{j\neq k} x_jp_jH_{k,j}+\sigma^2} \right)\geq 
    \eta_kx_k, ~ \forall k,
    \label{eq:load_constraints_after_sample} \\ 
      & \ \ \ \ \ \ \textsf{constraints } (\ref{eq:power_limits}),(\ref{eq:binary_constraint}),
  \end{align}
\end{subequations}
where
\begin{align}
&\eta_k=\frac{V_k (|\mathcal{D}_k|-|\tilde{\mathcal{D}_{k}}|)}{(T-T_0)B_k}.
\end{align}
It can be seen that the time budget is reduced from $T$ to $T-T_0$, since $T_0$ has been consumed during loss prediction. 
Accordingly, the required number of samples for transmission is reduced from 
$|\mathcal{D}_k|$ to $|\mathcal{D}_k|-|\tilde{\mathcal{D}_{k}}|$, since $\tilde{\mathcal{D}_{k}}$ has already been uploaded in the sampling stage.

Problem $\mathsf{P}_2$ involves nonlinear coupling among binary variables $\mathbf{x}$ and continuous variables $\mathbf{p}$. 
Furthermore, even if we relax $\mathbf{x}$ into a continuous variable, the resultant problem is still nonconvex due to potential multi-user interference.
Hence, existing optimization solvers such as Mosek are not applicable.
To tackle the challenge, in the following, we will propose a PAMM algorithm that solve $\mathsf{P}_2$ efficiently. 
It consists of an outer penalty alternating minimization (PAM) iteration and an inner majorization minimization (MM) iteration (thus we name it PAMM).
Below we present PAM and MM in detail.

\subsection{Penalty Alternating Minimization}
\label{sec:pmm}

To tackle the discontinuity in constraint \eqref{eq:binary_constraint}, we first relax the binary constraint $x_{k}\in\{0,1\}$ into an affine constraint $x_{k}\in [0,1]$, $\forall k$. 
However, the relaxation is generally not tight and the solution to the relaxed problem could be $0<x_{k}<1$. 
To promote a binary solution for the relaxed variable $\{x_{k}\}$, we augment the objective function with a penalty term as in \cite{rinaldi2009new}. 
Here, we adopt the following penalty function \cite{lucidi2010exact}:
\begin{align} \label{eq:penaltyterm1}
    \varphi_1(\mathbf x)=\frac{1}{\beta}\sum_{k=1}^K x_{k}(1-x_k),
\end{align}
where $\beta>0$ is the penalty parameter.

On the other hand, to tackle the nonconvex logarithm function in the constraint \eqref{eq:load_constraints_after_sample}, we introduce slack variables $\bm{\xi}=[\xi_{1},\cdots,\xi_{K}]^T$ such that
$\xi_{k}= x_kp_k, \forall k$.
This is equivalent to augmenting the objective function with the following penalty term:
\begin{align} \label{eq:penaltyterm}
    \varphi_2(\mathbf{x},\mathbf{p},\bm{\xi})=\gamma
    \sum_{k=1}^{K} \|\xi_{k}-x_kp_k\|^2,
\end{align}
where $\gamma$ is a sufficiently-large penalty parameter to ensure tight coupling between $\xi_{k}$ and $x_kp_k$.

Based on the above penalty functions, problem $\mathsf{P}_2$ is equivalently transformed into 
  \begin{align}
    &\mathsf{P}_3:\min_{\substack{\mathbf{x},\mathbf{p},\bm{\xi}}} 
      \quad  -\sum_{k=1}^{K} x_k \, 
    \tilde{\pi}_k(\mathcal{S}^{'},\tilde{\mathcal{D}}_k)
    +
\varphi_1(\mathbf x)  
+\varphi_2(\mathbf{x},\mathbf{p},\bm{\xi})
      \nonumber \\
    & \ \text{s.t.} \ \
      {\Phi}_{k}(\mathbf{x},\bm{\xi})\leq 0
    , \ 0\leq x_k \leq 1, \ \forall k,
     \nonumber \\ 
      & \ \ \ \ \ \
      0 \leq p_k \leq P_{\text{max}}, ~ \forall k, ~ \sum_{k=1}^K p_k  \leq P_{\text{sum}}, 
  \end{align}
where 
\begin{align}
  {\Phi}_{k}(\mathbf{x},\bm{\xi})
  =
  \eta_k x_k-\log_2 \left( 1 + \frac{\xi_{k}H_{k,k}}
	{
    \sum_{j\neq k} \xi_{j}H_{k,j}+\sigma^2} \right).
\end{align}

According to \cite[Proposition 1]{lucidi2010exact}, with the penalty term in (\ref{eq:penaltyterm1}), there exists an upper bound $\bar\beta>0$ such that for any $\beta\in[0,\bar\beta]$, $\mathrm{P}_{2}$ and $\mathrm{P_{3}}$ are equivalent with a proper choice of $\beta$ \cite{lucidi2010exact}. 

Now, variables $\{\mathbf{x},\bm{\xi}\}$ and $\mathbf{p}$ are independent in all constraints. Consequently, to solve $\mathrm{P}_3$, we can adopt alternating minimization that solves $\{\mathbf{x},\bm{\xi}\}$'s subproblem and $\mathbf{p}$'s subproblem iteratively. 
This procedure is guaranteed to converge to a local minimum solution of $\mathsf{P}_3$ according to 
\cite{abeck}. 

In particular, the method initializes a feasible $\mathbf{p}=\mathbf{p}^{[0]}$ (e.g., $\mathbf{p}^{[0]}=P_{\mathrm{sum}}/K$) and then solves the following sequence of optimization problems:
\begin{align}
    \mathrm{P}_{3a}^{[1]} \rightarrow \mathrm{P}_{3b}^{[1]} \rightarrow \mathrm{P}_{3a}^{[2]} \rightarrow \mathrm{P}_{3b}^{[2]} \cdots \label{iter}
\end{align}
At the $i$-th iteration of the PAM, given a fixed $\mathbf{p}=\mathbf{p}^{[i-1]}$, we solve for $\{\mathbf{x},\bm{\xi}\}$ as follows:
\begin{subequations}
  \begin{align}
   & \mathsf{P}_{3a}^{[i]}: \min_{\substack{\mathbf{x},\bm{\xi}}} 
      \quad -\sum_{k=1}^{K} x_k \, 
    \tilde{\pi}_k(\mathcal{S}^{'},\tilde{\mathcal{D}}_k)
    +
\varphi_1(\mathbf x)
    \nonumber\\
    &\quad\quad \
    +
    \gamma
    \sum_{k=1}^{K} \|\xi_{k}-x_kp_k^{[i-1]}\|^2
    , 
      \label{P3a_1} \\
    & \ \text{s.t.} \ \
 {\Phi}_{k}(\mathbf{x},\bm{\xi})\leq 0, \ 0\leq x_k \leq 1, ~ \forall k.
    \label{P3a_2} 
  \end{align}
\end{subequations}
Denoting the solution of $\mathrm{P}_{3a}^{[i]}$ as $\{\mathbf{x}^*,\bm{\xi}^*\}$, we then set $\{\mathbf{x}^{[i]}=\mathbf{x}^*,\bm{\xi}^{[i]}=\bm{\xi}^*\}$, and solve for $\mathbf{p}$ as follows:
\begin{subequations}
  \begin{align}
    \mathsf{P}_{3b}^{[i]}: & \min_{\substack{\mathbf{p}}} 
      \quad  \sum_{k=1}^{K} \|\xi_{k}^{[i]}-x_k^{[i]} p_k\|^2
    ,  \\
    & \ \text{s.t.}  \ \
      0 \leq p_k \leq P_{\text{max}}, ~ \forall k, ~ \sum_{k=1}^K p_k  \leq P_{\text{sum}}.
  \end{align}
\end{subequations}
Denoting the optimal solution to $\mathrm{P}_{3b}^{[i]}$ as $\mathbf{p}^*$, we then set 
$\mathbf{p}^{[i]}=\mathbf{p}^*$. 
This completes one iteration round. 
By setting $i\leftarrow i+1$, we can proceed to solve the problem $\mathrm{P}_{3a}^{[i+1]}$, and the process continues until the $i$ reaches the maximum number of iterations, i.e., $i=\mathcal{I}$.

\subsection{Majorization Minimization}
\label{sec:mm}

In the iterative procedure \eqref{iter}, $\mathrm{P}_{3b}^{[i]}$ is a linearly constrained quadratic optimization problem and can be optimally solved by off-the-shelf software (e.g., CVXPY). 
However, $\mathrm{P}_{3a}^{[i]}$ is a nonconvex optimization problem, since the function $\varphi_1$ in \eqref{P3a_1} and the function $\Phi_k$ in \eqref{P3a_2} are nonconvex. 
To tackle these functions, we propose to leverage the framework of MM, which constructs a sequence of upper bounds $\{\widehat{\varphi}_1\}$ on $\varphi_1(\mathbf x)$ and replaces $\varphi_1(\mathbf x)$ in $\rm{P}3$ with $\{\widehat{\varphi}_1\}$ to obtain the surrogate problems. 
Similarly, MM would also construct upper bounds $\{\widehat{\Phi}_k\}$ on 
$\Phi_k$ and replace $\Phi_k$ in $\rm{P}3$ with $\{\widehat{\Phi}_k\}$.
More specifically, given any feasible solution $\{\mathbf{x}^\star,\bm{\xi}^\star\}$ to $\mathrm{P}3$, we define surrogate functions
\begin{align}
\widehat{\varphi}_1(\mathbf x|\mathbf x^\star ) &= 
\sum_{k=1}^K 
\left(
\frac{1}{\beta}x_{k}-\frac{2}{\beta}x_{k}^{\star}x_{k}+\frac{1}{\beta}x_{k}^{\star^2}
\right),
\end{align}
\begin{align}
&\widehat{\Phi}_{k}(\mathbf{x},\bm{\xi}|\bm{\xi}^\star)
=\eta_kx_k-\frac{1}{\mathrm{ln}2}
\Bigg[
\mathrm{ln}\left(\sum_{l=1}^K\frac{H_{k,l}\xi_{l}}{\sigma^2}+1\right)
\nonumber\\
&\quad
{}
-
\mathrm{ln}\left(\sum_{l=1,l\neq k}^K\frac{H_{k,l}\xi^\star_{l}}{\sigma^2}+1\right)-\left(\sum_{l=1,l\neq k}^K\frac{H_{k,l}\xi^\star_{l}}{\sigma^2}+1\right)^{-1}
\nonumber
\\&\quad
{}
\times\left(\sum_{l=1,l\neq k}^K\frac{H_{k,l}\xi_{l}}{\sigma^2}+1\right)
+1
\Bigg]. \label{Phi}
\end{align}
and the following proposition can be established.

\begin{algorithm}[!t]
    \caption{PAMM for EGS Full Transmission}
    \label{alg:PAMM}
        \begin{algorithmic}[1]
            \State \textbf{Input} $\{H_{k,j}\}$, $\sigma^2$, $T$, $T_0$, $\mathcal{D}$, $\tilde{\mathcal{D}}$, $\mathcal{S}^{'}$, $P_{\text{sum}}$, $P_{\text{max}}$, $B_k$, $V_k$.
            \State \textbf{Initialize} approximate GS loss $\tilde{\pi}_k$ according to \eqref{appgs}.
            \State \textbf{Initialize} $\beta$, $\gamma$, and $\{\mathbf{p}^{[0]},\mathbf{x}^{[0]},\bm{\xi}^{[0]}\}$, and set $i=0$.
            \State \textbf{Repeat}
            \State \ \ \  Set $n=0$ and $(\mathbf{x}[0],\bm{\xi}[0])=(\mathbf{x}^{[n]},\bm{\xi}^{[0]})$:
            \State \ \ \ \ \ \  \textbf{Repeat}
            \State \ \ \ \ \ \ \ \ \  Solve $\mathsf{P}_{3a}^{[i]}[n+1]$ and obtain $\{\mathbf{x}^*,\bm{\xi}\}$.
            \State \ \ \ \ \ \ \ \ \  Update $\{\mathbf{x}{[n+1]}=\mathbf{x}^*,
            \bm{\xi}{[n+1]}=\bm{\xi}^*\}$.
            \State \ \ \ \ \ \ \ \ \  Update $n\leftarrow n+1$.
            \State \ \ \ \ \ \   \textbf{Until} $n=\mathcal{J}$.
            \State \ \ \ \ \ \  Update $(\mathbf{x}^{[i+1]},
            \bm{\xi}^{[i+1]})=(\bm{\xi}[\mathcal{J}],\mathbf{x}[\mathcal{J}])$.
            \State \ \ \   Solve $\mathsf{P}_{3b}^{[i]}$ using CVXPY and obtain $\mathbf{p}^*$.  
            \State \ \ \   
            Update $\mathbf{p}^{[i+1]}\leftarrow \mathbf{p}[\mathcal{J}]$
             \State \ \ \   
            Update $i \leftarrow  i+1$.
            \State \textbf{Until} $i=\mathcal{I}$ and the converged solution is $\{\mathbf{x}^\diamond,\mathbf{p}^\diamond,\bm{\xi}^\diamond\}$.
            \State \textbf{Output} $\{\mathbf{x}^\diamond,\mathbf{p}^\diamond\}$.
        \end{algorithmic}
\end{algorithm}

\begin{proposition}
The functions $\{\widehat{\varphi}_1,\widehat{\Phi}_k\}$ satisfy the following:

\noindent(i) Upper bound: 
$$\widehat{\varphi}_1(\mathbf x|\mathbf x^\star )\geq \varphi_1(\mathbf{x}), \ \widehat{\Phi}_{k}(\mathbf{x},\bm{\xi}|\bm{\xi}^\star)\geq 
{\Phi}_{k}(\mathbf{x},\bm{\xi})$$

\noindent(ii) Convexity: $\widehat{\varphi}_1(\mathbf x|\mathbf x^\star )$ is convex in $\mathbf{x}$, and $\widehat{\Phi}_{k}(\mathbf{x},\bm{\xi}|\bm{\xi}^\star)$ is convex in $(\mathbf{x},\bm{\xi})$.

\noindent(iii) Local equivalence: 
\begin{align}
    \widehat{\varphi}_1(\mathbf x^\star|\mathbf x^\star )&= \varphi_1(\mathbf{x}^\star),
    \nonumber\\
    \nabla_{\mathbf{x}}\widehat{\varphi}_1(\mathbf x^\star|\mathbf x^\star ) &= \nabla_{\mathbf{x}}\varphi_1(\mathbf{x}^\star),
        \nonumber\\
    \widehat{\Phi}_{k}(\mathbf{x}^\star,\bm{\xi}^\star|\bm{\xi}^\star) &= {\Phi}_{k}(\mathbf{x}^\star,\bm{\xi}^\star),
    \nonumber\\
    \nabla_{(\mathbf{x},\bm{\xi})}\widehat{\Phi}_{k}(\mathbf{x}^\star,\bm{\xi}^\star|\bm{\xi}^\star) &= \nabla_{(\mathbf{x},\bm{\xi})} {\Phi}_{k}(\mathbf{x}^\star,\bm{\xi}^\star).
\end{align}

\end{proposition}
\begin{proof}
See Appendix.
\end{proof}
With part (i) of \textbf{Proposition 1}, an upper bound can be directly obtained if we replace the functions $\{\varphi_1,\Phi_k\}$ by $\{\widehat{\varphi}_1,\widehat{\Phi}_k\}$ around a feasible point.
However, a tighter upper bound can be achieved if we treat the obtained solution as another feasible point and continue to construct the next-round surrogate function.
In particular, assuming that the solution at the $n^{\mathrm{th}}$ iteration is given by $\{\mathbf{x}{[n]},\mathbf{\xi}{[n]}\}$, the following problem is considered at the $(n+1)^{\mathrm{th}}$ iteration:
\begin{subequations}
  \begin{align}
   & \mathrm{P}_{3a}^{[i]}[n+1]: \min_{\substack{\mathbf{x},\bm{\xi}}} 
      \quad -\sum_{k=1}^{K} x_k \, 
    \tilde{\pi}_k(\mathcal{S}^{'},\tilde{\mathcal{D}}_k)
    +
\widehat{\varphi}_1(\mathbf x|\mathbf x{[n]})
    \nonumber\\
    &\quad\quad \
    +
    \gamma
    \sum_{k=1}^{K} \|\xi_{k}-x_kp_k^{[i-1]}\|^2
      \label{P3a_1} \\
    & \ \text{s.t.} \ \
 \widehat{\Phi}_{k}(\mathbf{x},\bm{\xi}|\bm{\xi}{[n]})\leq 0, \ 0\leq x_k \leq 1, ~ \forall k
    \label{P3a_2}
  \end{align}
\end{subequations}

Based on part (ii) of \textbf{Proposition 1}, the problem $\mathrm{P}_{3a}^{[i]}[n+1]$ is convex and can be solved by off-the-shelf software packages (e.g., Mosek) for convex programming.
Denote its optimal solution as
$\{\mathbf{x}^*,\bm{\xi}\}$. 
Then we set
$\mathbf{x}^{[n+1]}=\mathbf{x}^*$ and $\{\bm{\xi}{[n+1]}=\bm{\xi}^*\}$, such that the process repeats with solving the problem $\mathrm{P}_{3a}^{[i]}[n+2]$.
According to part (iii) of \textbf{Proposition 1} and \cite[Theorem 1]{{sun2016majorization}}, every limit point of the sequence
$(\mathbf{x}{[0]},\bm{\xi}{[0]}),(\mathbf{x}{[1]},\bm{\xi}{[1]}),\cdots)$ is a stationary point to $\mathrm{P}_{3a}^{[i]}$ as long as the starting point $\{\mathbf{x}{[0]},\bm{\xi}{[0]}\}$ is feasible to $\mathrm{P}_{3a}^{[i]}$.

\subsection{Algorithm and Complexity}

The entire procedure of the PAMM method is summarized in Algorithm~\ref{alg:PAMM}.
In terms of computational complexity, $\mathrm{P}_{3a}^{[i]}[n+1]$ involves $2K$ variables.
Therefore, the worst-case complexity for solving $\mathrm{P}_{3a}^{[i]}[n+1]$ is
$\mathcal{O}\big((2K)^{3.5}\big)$.
Consequently, the total complexity for solving $\mathrm{P}_{3a}^{[i]}$ is $\mathcal{O}\big(\mathcal{J}\,(3K)^{3.5}\big)$, where $\mathcal{J}$ is the number of iterations needed for the MM algorithm to converge.
On the other hand, since $\mathrm{P}_{3b}^{[i]}[n+1]$ involves $K$ variables, the worst-case complexity for solving $\mathrm{P}_{3b}^{[i]}$ is
$\mathcal{O}\big(K^{3.5}\big)$.
Based on the above analysis, the total complexity for solving $\mathrm{P}_{3}$ exploiting PAMM is 
$\mathcal{O}\big(\mathcal{J}\,(3K)^{3.5}+K^{3.5}\big)$. 

To further validate its practical scalability, we profiled the execution time versus the number of clients $K$. The result is shown in Fig.~\ref{fig:execution_time_vs_clients_number}, and we have the following key observations.

\begin{itemize}
    \item At $K=5$ (the setting adopted in our main experiments),  the average solving time is approximately $18$\, seconds.
    \item If we scale up to $K=11$, the solving time becomes $43$\,s, which still remains feasible for real applications.
    \item If $K=18$, the solving time exceeds $320$\,s, which requires faster optimization. This can be realized via hierarchical grouping based on spatial correlation and GPU-accelerated computation.
    \item The function of execution time w.r.t. the number of clients exhibits a polynomial relationship.
\end{itemize}
The above results demonstrate that the proposed PAMM algorithm is computationally feasible for a tens of clients, a typical scale in edge computing.

\begin{figure}[!t]
    \centering
    \begin{tabular}{@{}c@{\hspace{0.01\textwidth}}c@{}}
    \begin{minipage}[t]{0.48\columnwidth}
        \centering
        \includegraphics[width=\linewidth, keepaspectratio]{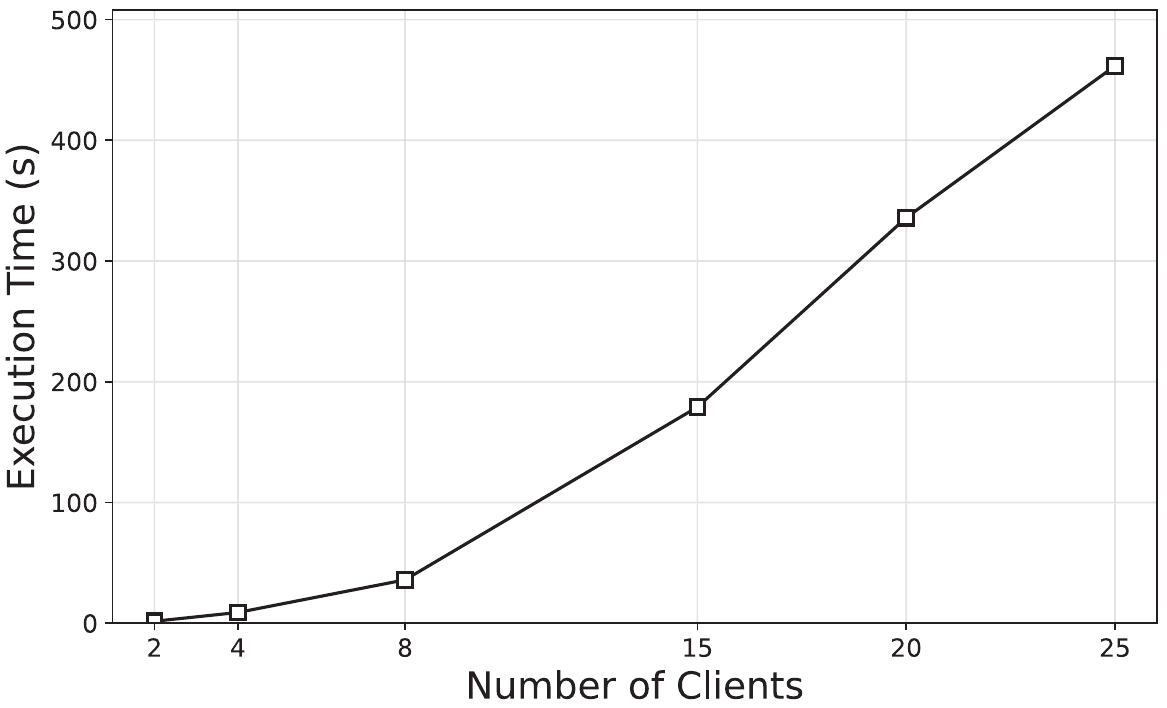}
        \caption{Execution time vs. $K$.}
        \label{fig:execution_time_vs_clients_number}
    \end{minipage} &
    \begin{minipage}[t]{0.48\columnwidth}
        \centering
        \includegraphics[width=\linewidth, keepaspectratio]{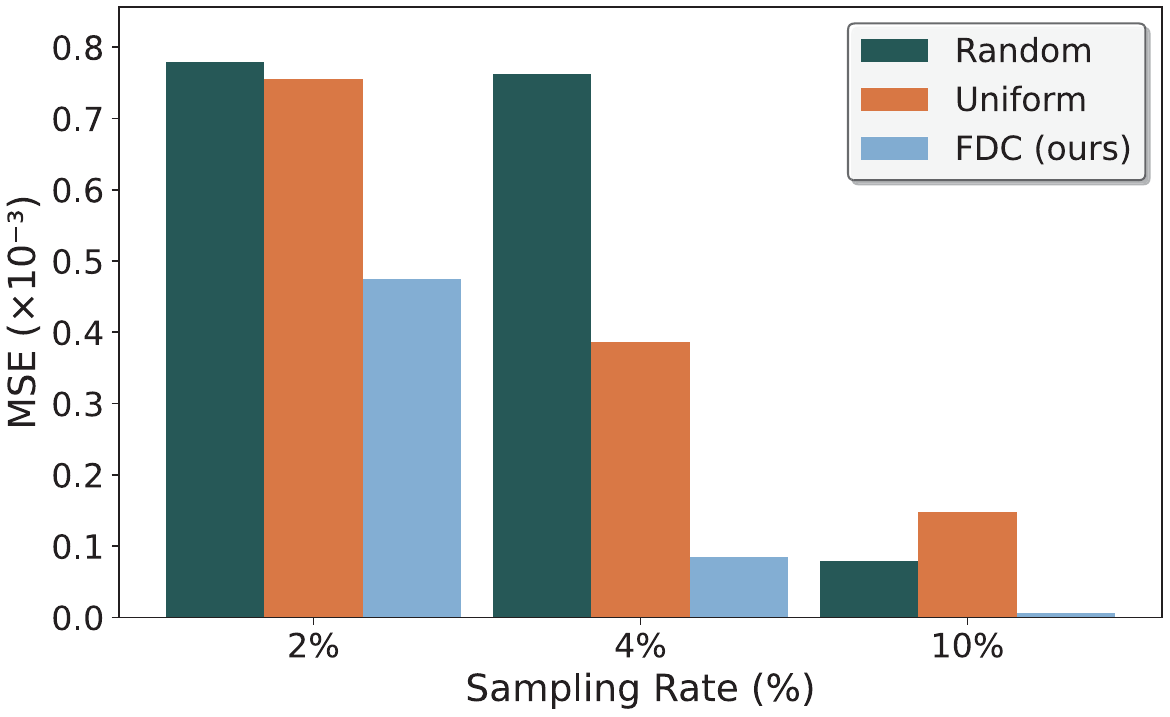}
        \caption{MSE vs. $\rho$.}
        \label{prediction_different_rate_Bar.pdf}
    \end{minipage}
    \end{tabular}
\end{figure}

\begin{figure*}[t]
    \centering
    \setlength{\tabcolsep}{0pt} 
    \begin{tabular}{@{}c@{\hspace{0.5mm}}c@{\hspace{0.5mm}}c@{}} 
    \begin{minipage}[t]{0.32\textwidth}
        \centering
        \includegraphics[width=0.99\linewidth]{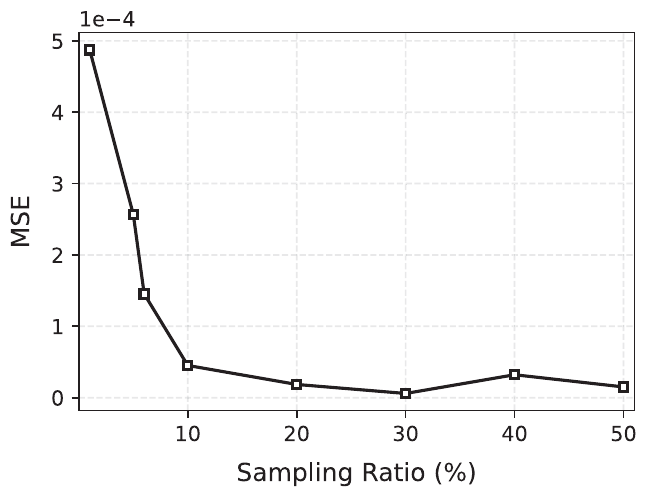}
        \captionof{figure}{MSE vs sampling ratio $\rho$.}
        \label{fig:mse_predicted_loss_and_true_loss}
    \end{minipage} &
    \begin{minipage}[t]{0.31\textwidth}
        \centering
        \includegraphics[width=0.99\linewidth]{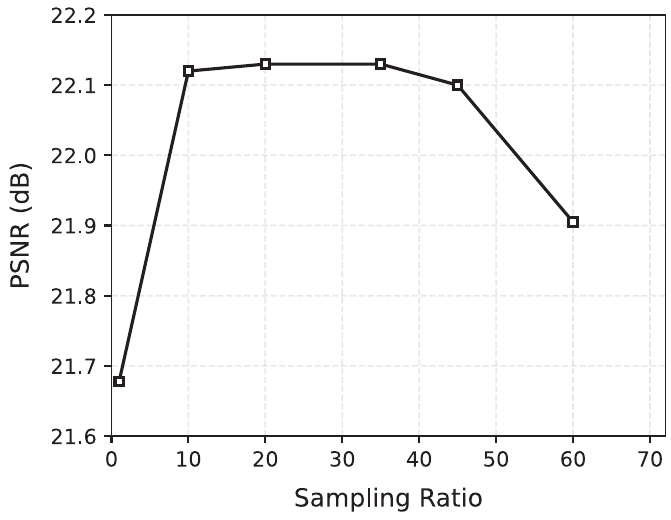}
        \captionof{figure}{PSNR vs sampling ratio $\rho$.}
        \label{fig:psnr_under_sampling_ratio}
    \end{minipage} &
    \begin{minipage}[t]{0.35\textwidth}
        \centering
        \includegraphics[width=0.99\linewidth]{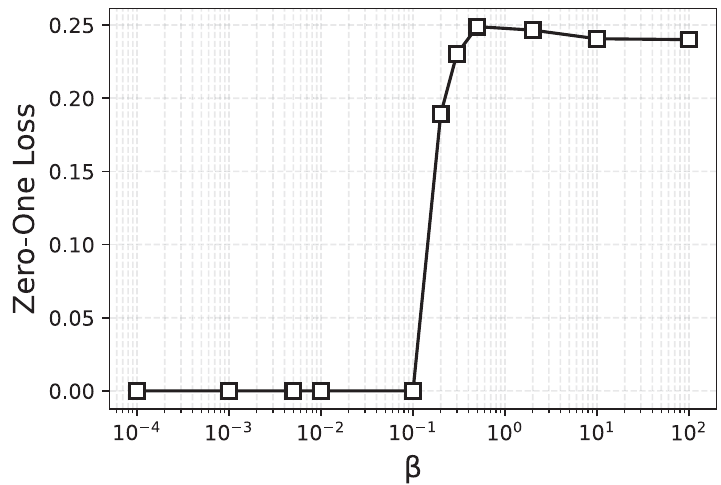}
        \captionof{figure}{Zero-one loss vs $\beta$ ($K=5$).}
        \label{fig:zero_one_loss}
    \end{minipage}
    \end{tabular}
\end{figure*}

\section{Experiments}\label{section6}

\begin{figure}[t!]
    \centering
    \begin{subfigure}[t]{0.32\linewidth}
        \centering
        \includegraphics[width=0.95\linewidth]{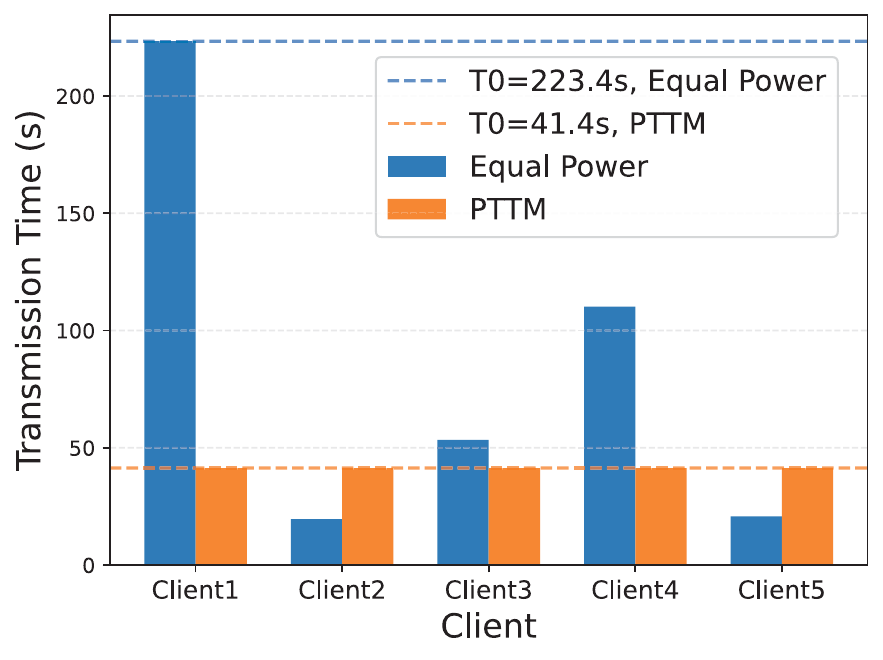}
        \caption{Pilot time cost}
        \label{fig:pttm_ep_time}
    \end{subfigure}
    \hspace{-0.5em}
    \begin{subfigure}[t]{0.32\linewidth}
        \centering
        \includegraphics[width=0.95\linewidth]{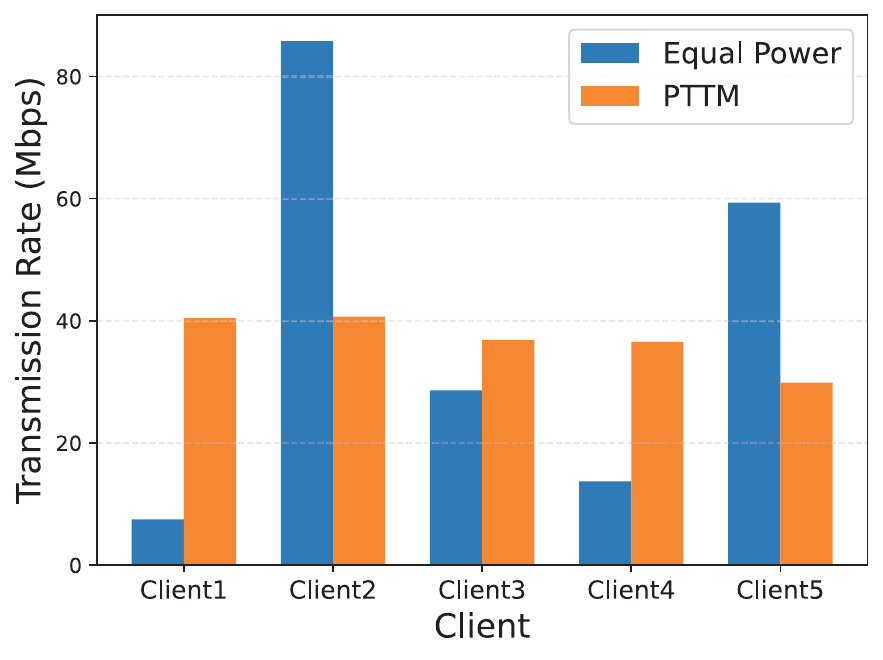}
        \caption{Rate profile}
        \label{fig:pttm_ep_rate}
    \end{subfigure}
    \hspace{-0.5em}
    \begin{subfigure}[t]{0.32\linewidth}
        \centering
        \includegraphics[width=0.95\linewidth]{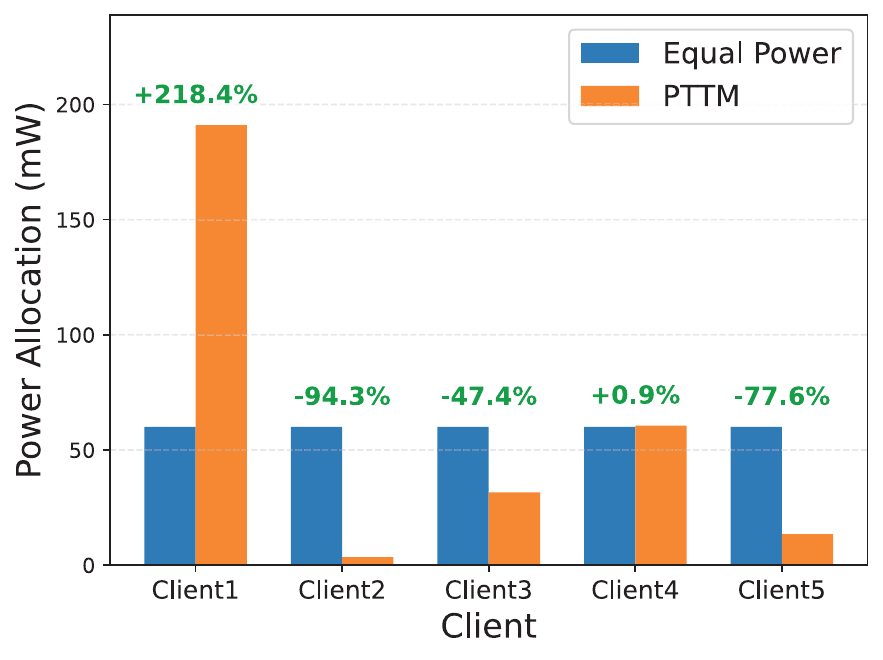}
        \caption{Power profile}
        \label{fig:pttm_ep_power}
    \end{subfigure}
    \caption{Performance comparison between the proposed PTTM and equal power algorithms.}
    \label{fig:first_stage_comp}
\end{figure}

\begin{figure}[!t]
  \centering 
  \includegraphics[width=0.45\columnwidth]{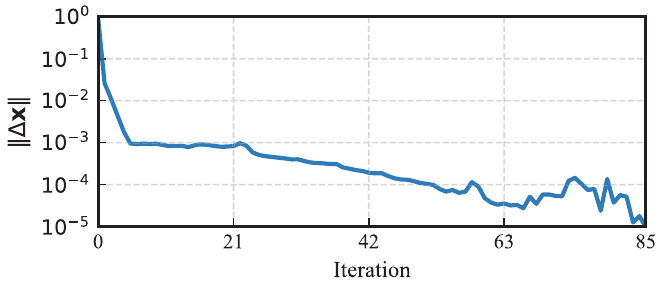} 
  \label{fig:convergence_x}
  \includegraphics[width=0.45\columnwidth]{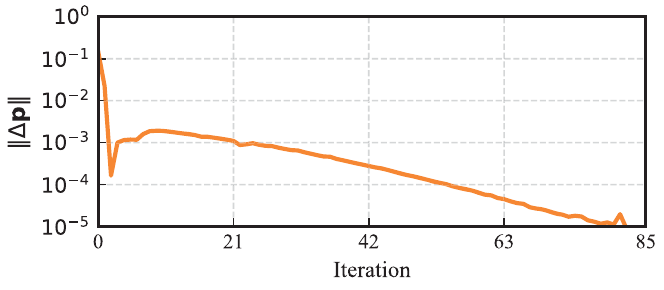} 
    \label{fig:convergence_p}
  \caption{$\|\Delta \mathbf{x}\|$ and $\|\Delta \mathbf{p}\|$ versus n.} 
  \label{fig:convergence}
\end{figure}

We implement the proposed STT-GS system with FDC, PTTM, and PAMM algorithms in Python based on the 3D GS project \cite{kerbl20233d}. The system is deployed on an Ubuntu workstation equipped with a 3.75 GHz AMD EPYC 16-core CPU and an NVIDIA RTX 4090 GPU.
We consider the case of $N=64$ and $K=5$, where all clients are randomly distributed within a $100 \times 100\,\text{m}^2$ area, while the edge server is located at the center $(0, 0)$.
Then we compute the distances $\{d_k\}$ for all $K$ based on the locations of clients and server.
The pathloss exponent $\alpha = 3$ and the shadow fading $\omega_k = -20$\,dB. 
The noise power $\sigma^2 = -120 \, \text{dBm}$.
The total time budget is $T=350$\,s.
The total power is set to $P_{\text{sum}}=300$\,mW and the maximum power is $P_{\text{max}}=200$\,mW according to \cite{yu2004iterative,zhang2024efficient,schubert2004solution,al2011achieving,beck20141,wang2020learning,wang2020machine,yoo2019learning}.
The system bandwidth is set to $B=10\, \text{MHz}$.
Simulation parameters are summarized in Table~\ref{TableP}, where the channel parameters are set according to \cite{wang2020angle}.

\begin{table}[!t]
\centering
\caption{Simulation Parameters}
\label{TableP}
\scalebox{0.8}{
\begin{tabular}{|c|c|c|}
\hline
\textbf{Parameter} & \textbf{Description} & \textbf{Value} \\
\hline
$N$ & Number of antennas at server & 64 \\
\hline
$K$ & Total number of clients & 5 \\
\hline
$T$ & Time budget & $350\,\textrm{s}$ \\
\hline
$P_\textrm{max}$ & Maximum transmit power of each client & $200\,\textrm{mW}$ \\
\hline
$P_\textrm{sum}$ & Sum transmit power of all clients & $300\,\textrm{mW}$ \\
\hline
$\sigma^2$ & Noise power & -100\,dBm \\
\hline
$ K_\textrm{Ric}$ & The Rician K-factor & -26\,dB \\
\hline
\end{tabular}
}
\end{table}

Experiments are conducted exploiting the rubble-pixsfm LAE dataset \cite{turki2022mega}, which consists of $1680$ real-world images.
The dataset is divided into $6$ batches, where $5$ batches are adopted as training data at the $5$ clients and the remaining batch is reserved for testing.
The data volumes at clients 1--5 are: $2091.26$\,MB, $2103.93$\,MB, $1906.72$\,MB, $1891.08$\,MB, and $1544.17$\,MB, respectively. 
We randomly choose one of the $5$ batches to train the initial model $\mathcal{S}'$ at the edge server.

\begin{figure*}[!t]    \centering
\includegraphics[width=0.95\textwidth]{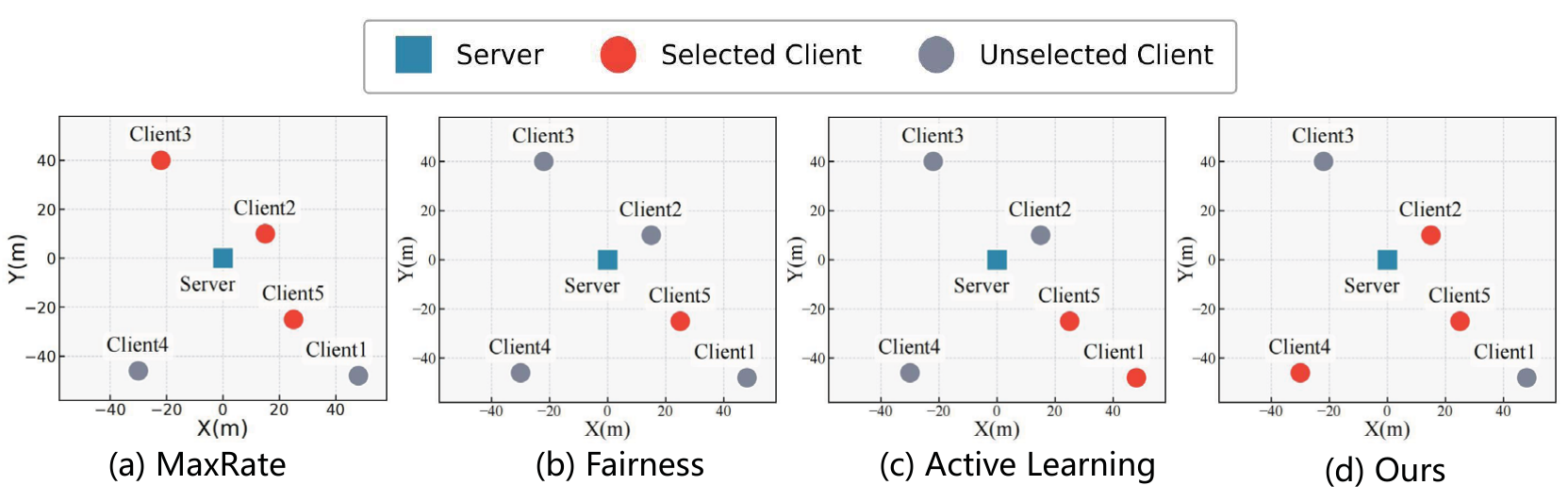}
    \caption{Comparative analysis of client selection strategies: 
(a) \textbf{MaxRate} prioritizes proximal clients $\{2,3,5\}$ for channel efficiency; 
(b) \textbf{Fairness} selects $\{2,5\}$ for equitable resource distribution; 
(c) \textbf{Active Learning} targets $\{1,5\}$ for maximum information gain; 
(d) \textbf{Proposed STT-GS} optimizes $\{2,4,5\}$ by balancing view contributions and channel conditions.}
    \label{fig:client-selection}
\end{figure*}

\begin{figure*}[!t] 
    \centering
    \setlength{\tabcolsep}{2pt}  
    \begin{tabular}{ccc}
    \includegraphics[clip,width=0.98\textwidth]{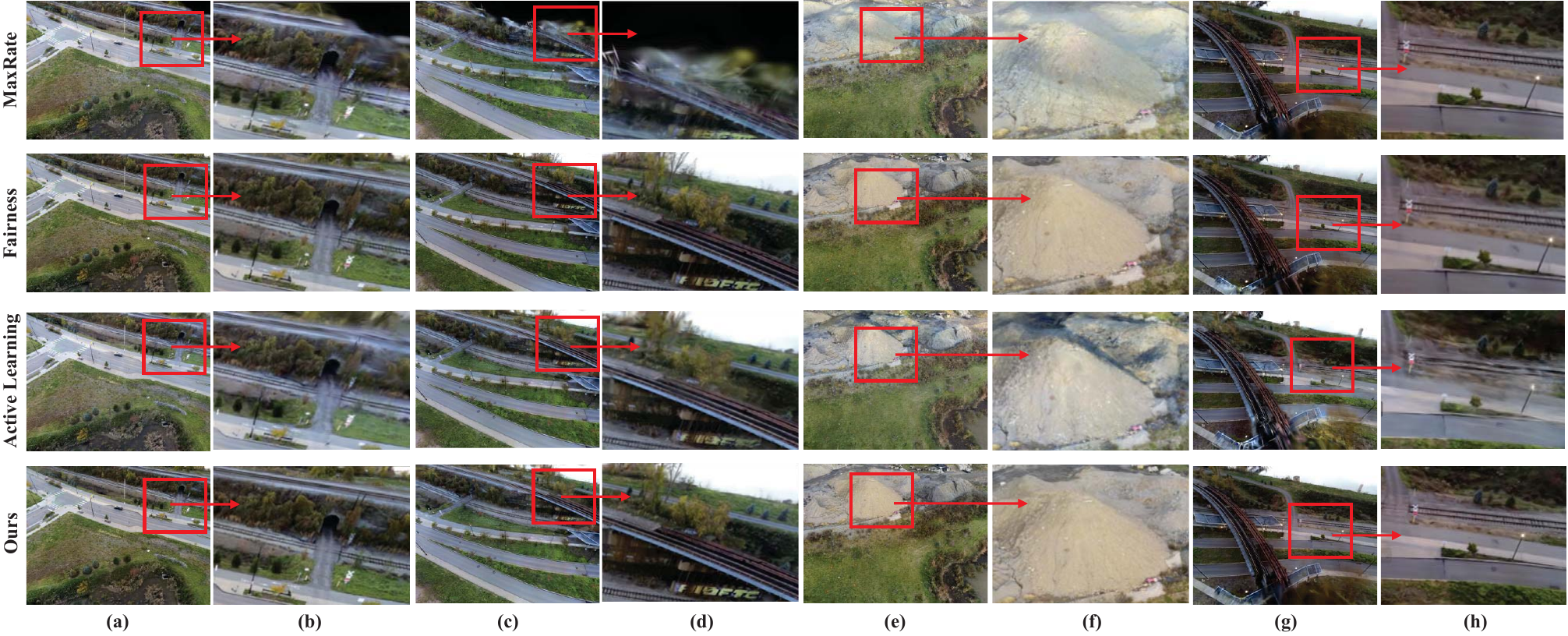} \\
    \end{tabular}
    \caption{Comparison of rendering qualities in four views of the rubble-pixsfm dataset.}
    \label{fig:visual_comparison}
\end{figure*}

\begin{figure*}[t] 
  \centering
  \begin{subfigure}[t]{0.32\textwidth}
    \centering
    \includegraphics[width=\textwidth]{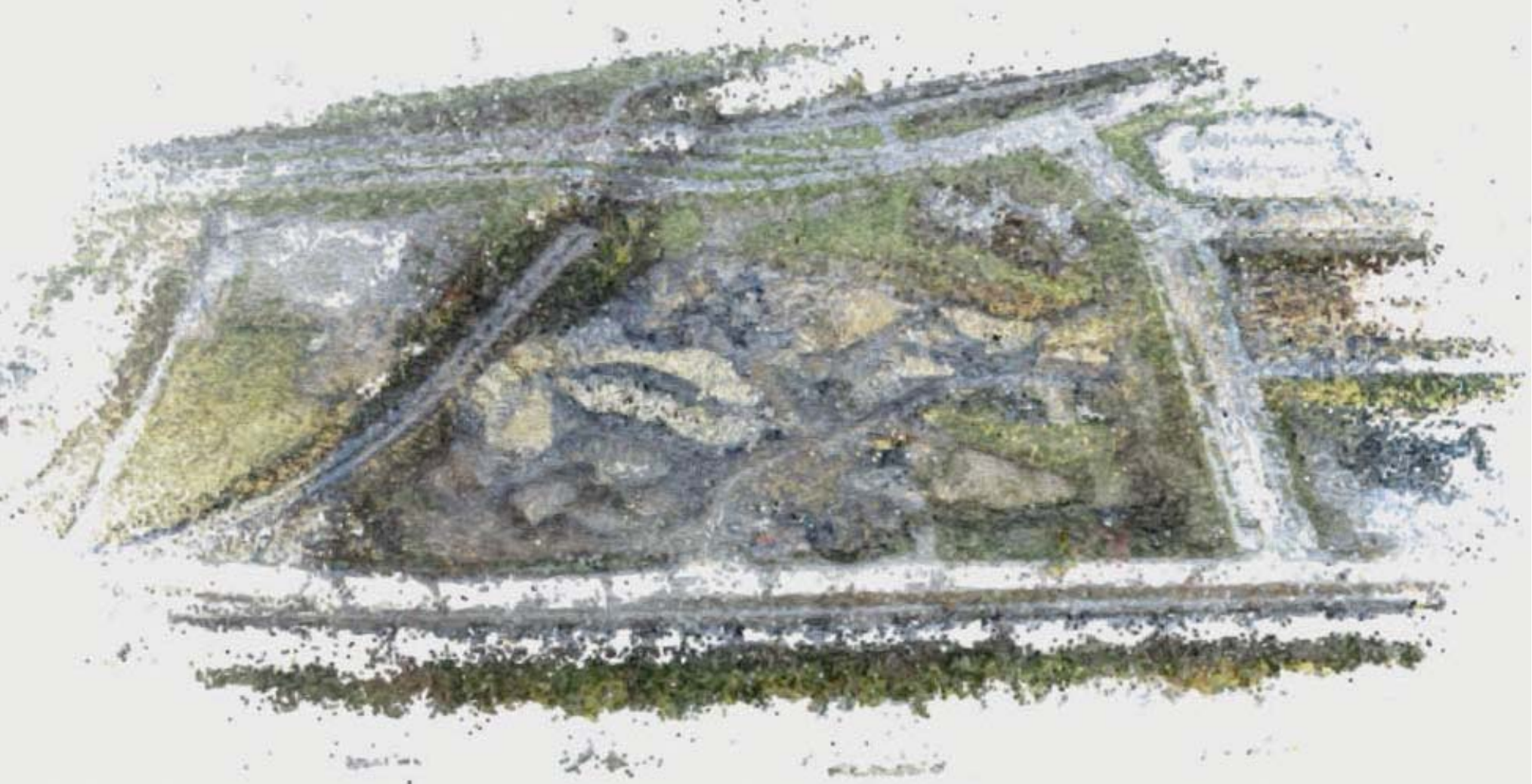}
    \caption{Point cloud map}
    \label{fig:point_cloud}
  \end{subfigure}
  \hfill
  \begin{subfigure}[t]{0.32\textwidth}
    \centering
    \includegraphics[width=\textwidth]{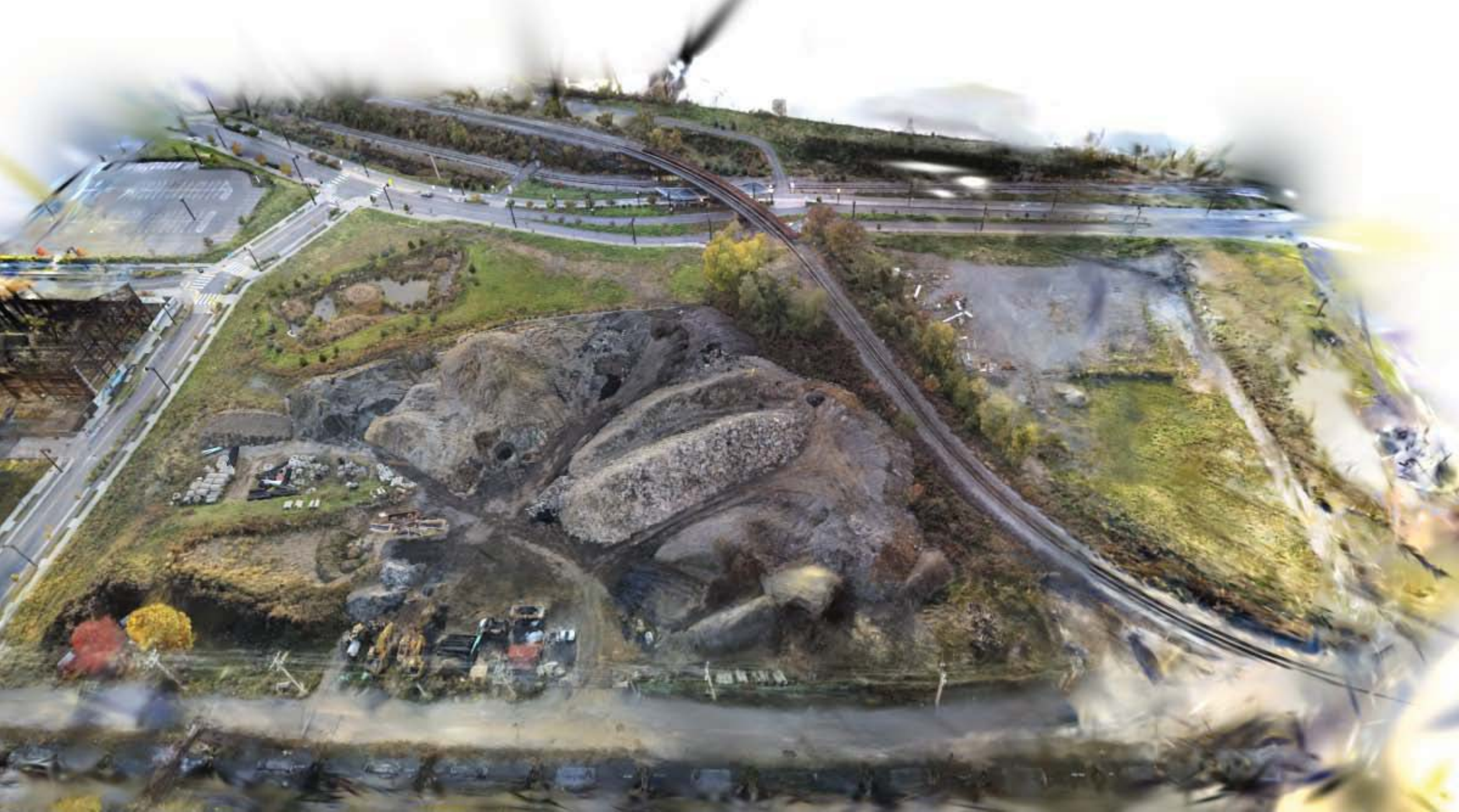}
    \caption{3D GS model (view 1)}
    \label{fig:3dgs_view1}
  \end{subfigure}
  \hfill
  \begin{subfigure}[t]{0.30\textwidth}
    \centering
    \includegraphics[width=\textwidth]{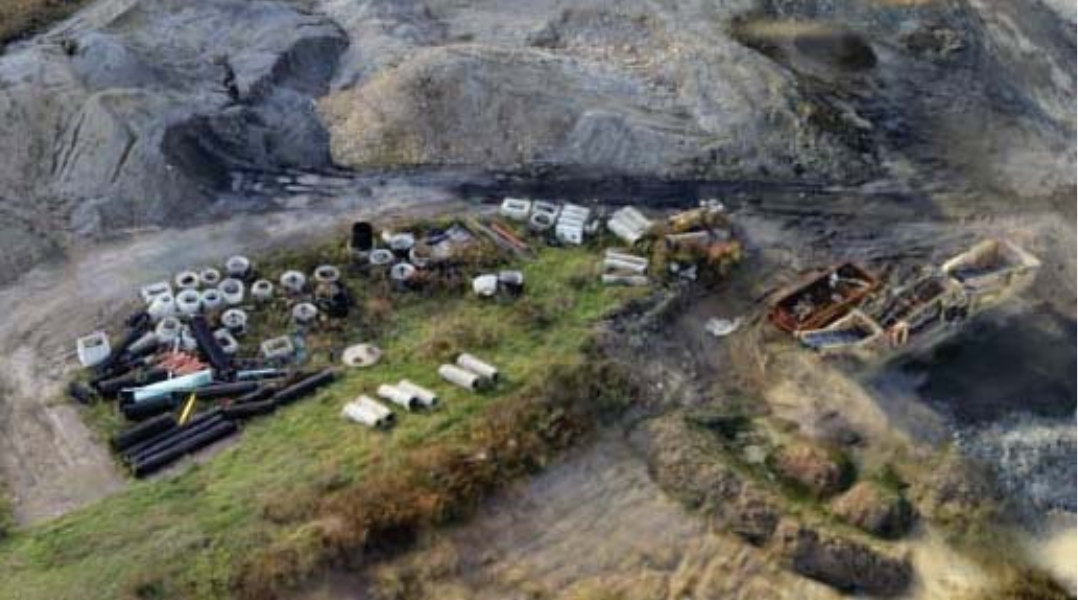}
    \caption{3D GS model (view 2)}
    \label{fig:3dgs_view2}
  \end{subfigure}
  
  \caption{Visualization of the proposed 3D GS model.}
  \label{fig:3d_reconstruction}
\end{figure*}

\subsection{Implementation Details}

\subsubsection{GS Models}

Our implementation builds directly on the open-source 3D GS repository by Kerbl \emph{et al.} \cite{kerbl20233d}, specifically utilizing the \texttt{3dgs-accel} branch of the \texttt{diff-gaussian-rasterization} submodule for accelerated differentiable rendering. 

\subsubsection{Training Configurations}

Hyperparameters and training settings are as follows:
\begin{itemize}
    \item \textbf{Optimizer}: Adam
    \item \textbf{Learning Rate}:
    \begin{itemize}
        \item Position: Initial rate $1.6\times10^{-4}$, exponentially decaying to $1.6\times10^{-6}$ with a discount factor of 0.01
        \item Rotation: $1.0\times10^{-3}$
        \item Scaling: $5.0\times10^{-3}$
        \item Opacity: $2.5\times10^{-2}$
        \item Spherical harmonics features: $2.5\times10^{-3}$
        \item Exposure parameters: Initial rate $1.0\times10^{-2}$, decaying to $1.0\times10^{-3}$ after specified delay steps
    \end{itemize}
    \item \textbf{Batch Size}: 1 (single camera viewpoint per iteration)
\end{itemize}

\subsubsection{Initial Model $\mathcal{S}'$}

The initial model $\mathcal{S}'$ was pretrained on the third partition of the rubble-pixsfm dataset used in the current experiments.
A sparse 3D point cloud was first reconstructed using COLMAP \cite{schonberger2016structure} via feature extraction, image matching, and triangulation from the RGB images in the selected partition. This SfM-based reconstruction provided the initial 3D keypoints for Gaussian placement.
Subsequently, the 3DGS model 
 $\mathcal{S}'$ was trained using the official implementation \cite{kerbl20233d}, with isotropic Gaussians initialized at the reconstructed keypoints. Training was performed for 40,000 iterations with a spatial resolution factor of 8, under standard settings (Adam optimizer, mixed precision). No client-specific weighting or adaptive sampling was applied during this pretraining phase.

\subsection{Evaluation of EGS Loss Prediction}

First, we conduct experiments to evaluate the performance of FDC in Algorithm 1. Fig.~\ref{prediction_different_rate_Bar.pdf} illustrates the mean squared errors (MSEs) between ground-truth and predicted losses, under different sampling ratios.
The proposed FDC achieves high prediction accuracy even at low sampling rates, \textcolor{black}{as the sampling procedure accounts for the significance of each image within the feature domain}. Consequently, it consistently achieves the smallest MSEs among all the simulated schemes across all sampling ratios (i.e., 2\%, 4\%, 10\%).

To obtain deeper insights into the impact of sampling ratios, Fig.~\ref{fig:mse_predicted_loss_and_true_loss} shows the MSE between predicted and actual rendering losses for at larger sampling ratios.
It can be seen that the prediction loss decreases as the sampling ratio increases. 
This implies that the distributions of $\tilde{\mathcal{D}}_k$ and $\mathcal{D}_k$ become closer with more pilot data. 
Furthermore, with our FDC, the prediction accuracy becomes close to zero when $\rho\geq10\%$. 
This demonstrates that low pilot overhead is achievable by using the proposed FDC method.

Fig. \ref{fig:psnr_under_sampling_ratio} presents the Peak Signal-to-Noise Ratio (PSNR) on the test dataset under different sampling ratios $\rho$. At $\rho < 10\%$, PSNR is low due to inaccurate loss prediction and suboptimal client selection. As $\rho$ increases to 10\%, PSNR reaches its peak and remains stable until $\rho=45\%$. 
However, if $\rho > 45\%$, PSNR decreases due to excessive pilot overhead in the first transmission phase, which reduces the time for the second-phase data transmission. 
This implies that there exists a trade-off relationship between sampling ratio and GS performance, and a proper sampling ratio can be determined by cross-validation.
Based on these experiments, we set $\rho=10\%$ in the subsequent experiments.

\begin{table}[!t]
    \centering
    \caption{Predicted Loss of Different Clients at $\rho_k=0.1$} 
    \label{Table-comparison}
    \vspace{0.1cm}
    \footnotesize 
    \setlength{\tabcolsep}{4pt} 
    \scalebox{1.1}{
    \begin{tabular}{lcccc}
        \toprule
        Client & \makecell{Ground Truth \\ Loss} & 
                 \makecell{Prediction \\ FDC (ours)} & 
                 \makecell{Prediction \\ Random} & 
                 \makecell{Prediction \\ Uniform} \\
        \midrule
        Client 1 & 0.38032 & 0.38159 & 0.39110 & 0.36226 \\
        Client 2 & 0.26530 & 0.26150 & 0.25352 & 0.26831 \\
        Client 3 & 0.02535 & 0.02561 & 0.02323 & 0.02381 \\
        Client 4 & 0.21635 & 0.21864 & 0.21256 & 0.19878 \\
        Client 5 & 0.31689 & 0.31429 & 0.30574 & 0.32664 \\
        \bottomrule
    \end{tabular}
    }
\end{table}

Then, we compare our FDC with two benchmark methods: \textbf{Equal-interval Sampling} and \textbf{Random Sampling} \cite{mayer2020adversarial}. The actual and predicted losses for each client at sampling ratio $\rho=10\%$ are summarized in Table~\ref{Table-comparison}.
The results indicate that the predicted losses of FDC for all clients closely match the ground-truth loss very well. 
Furthermore, FDC yields the smallest loss discrepancy compared to the other methods. 
Specifically, the prediction loss error with FDC is within 1.4\%, while other methods may involve over 10\% discrepancy \textcolor{black}{due to restricted feature diversity of the sampled data}. These findings confirm that FDC provides a more reliable and accurate loss prediction.

Then, we conduct a numerical experiment to verify the effectiveness of PTTM in Algorithm 2. 
The results of pilot transmission time, data rate, and power allocation are shown in Fig. \ref{fig:first_stage_comp}. It can be seen from Fig. \ref{fig:first_stage_comp}a that with PTTM, the first-stage sampling only costs $T_0=41.4\,$s and reserves over $300\,$s for the EGS full transmission stage. In contrast, with a naive equal power scheme, the pilot transmission stage requires $T_0=223.4\,$s, which consumes over $5\times$ cost compared to PTTM.
 \textcolor{black}{This demonstrates the benefit brought by joint considerations of sampling ratio, data volume, and communication conditions in PTTM.}

\subsection{Evaluation of EGS Full Transmission}

\begin{table}[!t]
\caption{Comparison of Various Image Quality Metrics}
\label{tab:rubble_metric_table}
\centering
\scalebox{0.75}{
\begin{tabular}{l c c c}
    \toprule
    Method & {$\uparrow$PSNR} & {$\uparrow$SSIM} & {$\downarrow$LPIPS} \\
    \midrule
    MaxRate         & 21.1503 & 0.71362 & 0.30658 \\
    Fairness        & 20.5004 & 0.68260 & 0.32154 \\
    Active Learning & 20.7619 & 0.71679 & 0.29347 \\
    Ours          & \textbf{22.1019} (+4.50\%) 
                    & \textbf{0.73347} (+2.78\%) 
                    & \textbf{0.29044} (-5.27\%) \\
    \bottomrule
    \multicolumn{4}{c}{Note: Performance gain is computed against the MaxRate scheme} 
\end{tabular}
}
\end{table}

Next, we conduct numerical experiments to validate the convergence of the proposed Algorithm 3 (i.e., PAMM).
Specifically, the variation between consecutive iterations $\|\Delta\mathbf{x}\|$ and $\|\Delta \mathbf{p}\|$ (with $ \Delta\mathbf{x}=\mathbf{x}^{[n]}-\mathbf{x}^{[n-1]}$ and 
$\Delta \mathbf{p} = \mathbf{p}[n] - \mathbf{p}[n-1]$) versus the number of iterations is shown in Fig.~\ref{fig:convergence}.
It can be seen that both $\|\Delta\mathbf{x}\|$ and $\|\Delta \mathbf{p}\|$ fall below $10^{-4}$ after $60$ iterations.
This demonstrates the convergence of PAMM.

To see the impact of the penalty parameter $\beta$, we have further conducted experiments, and the zero-one loss versus $\beta$ is shown in Fig.~\ref{fig:zero_one_loss}. 
When $\beta\geq 0.1$, the zero-one loss increases and becomes non-negligible, indicating that the binary penalty is too weak and the solution of $\mathbf{x}$ deviates significantly from binary values ($0/1$).
This aligns with the penalty function for the binary constraint  \eqref{eq:penaltyterm1}.
Consequently, $\beta$ needs to be smaller than the threshold $0.1$, such that the optimized zero-one loss tends to zero. 
However, $\beta$ cannot be too small. An excessively small $\beta$ weakens the optimization of the original objective. Therefore, we choose the largest $\beta$ that satisfies the condition of near-zero zero-one loss, which is $\beta = 0.1$ in our setting as shown in Fig.~\ref{fig:zero_one_loss}.

Then, we compare the proposed STT-GS consisting of FDC, PTTM, and PAMM algorithms against the benchmark schemes. 
We consider the following baselines: 1) \textbf{MaxRate}: EGS with water-filling power allocation for sum-rate maximization \cite{zhang2024efficient}; 2) \textbf{Fairness}: EGS with max-min fairness power allocation \cite{zheng2016wireless}; 3) \textbf{Active Learning}: EGS exploiting only active loss prediction, without considering channel conditions \cite{yoo2019learning}.

To conduct a quantitative analysis, we train the 3D GS model using the successfully transmitted client data for $30000$ iterations, and render images corresponding to the camera poses employed in the test dataset. We then compare these rendered images with ground truth images to compute PSNR, LPIPS, and SSIM metrics, as shown in Table~\ref{tab:rubble_metric_table}.
Overall, our method demonstrates superior performance across all metrics (i.e., PSNR, SSIM, LPIPS).
These results corroborates the fact that our method identifies the most valuable sensor data for 3D reconstruction.
More importantly, the benefit of incorporating the GS-oriented objective function outweighs the cost of adding a sampling stage.
This insight justifies the adoption of a two-stage sample-then-transmit pipeline in Section III.
Note that compared to the Active Learning scheme, our proposed scheme improves the PSNR and SSIM by 6.46\% and 2.33\%, respectively.
This finding highlights the benefit of adopting cross-layer optimization.

To gain further insights into the above results, Fig.~\ref{fig:client-selection} provides the user selection results of different schemes in a representative scenario. It can be seen that the MaxRate scheme allocates more powers to clients $\{2, 3, 5\}$, which are the closest users with the most favorable channel conditions. This allocation aligns with its throughput maximization objective, but may result in inefficient use of resources by prioritizing nearby users whose data may be of limited value. The Active Learning scheme selects clients $\{1, 5\}$, whose data yields the highest validation losses (thereby greatest information gains) for GS model $\mathcal{E}^{'}$ as shown in Table~\ref{Table-comparison}.
However, client $1$ is situated far from the server, leading to excessive communication costs. By excluding client $1$, our proposed method is able to connect more clients, specifically clients $\{2, 4, 5\}$, all of which posses high-quality GS data as seen in Table~\ref{Table-comparison}.
This corroborates the fact that our method achieves the best tradeoff between GS gains and communication costs, which also accounts for its superior performance in Table~\ref{Table-comparison}. It is also notable that the Fairness scheme only serves client $5$, since client $5$ possesses the minimum data volume.

To facilitate a more intuitive evaluation of the rendering performance, we also compare the visualization results from different views. 
Specifically, Fig.~\ref{fig:visual_comparison}a, Fig.~\ref{fig:visual_comparison}c, Fig.~\ref{fig:visual_comparison}e, and Fig.~\ref{fig:visual_comparison}g provide the rendered images from various camera poses, while Fig.~\ref{fig:visual_comparison}b, Fig.~\ref{fig:visual_comparison}d, Fig.~\ref{fig:visual_comparison}f, and Fig.~\ref{fig:visual_comparison}h provide the enlarged views of these scenes for detailed inspection.
Our method demonstrates superior performance for all views. 
For instance, the railway and the road of our method in Fig.~\ref{fig:visual_comparison}h are clear, while those of other schemes are blurred. 
To provide a global visualization of our method, the point cloud map and two bird's eye views of our GS model are provided in Fig.~\ref{fig:3d_reconstruction}. 
It can be seen from the point cloud map that our approach guarantees excellent geometry of the scenario.
Furthermore, as seen from the two bird's eye views, our method also guarantees excellent textures and semantics of the scenario.

\begin{figure*}[!t]
    \centering
    \begin{subfigure}[b]{0.33\textwidth}
        \includegraphics[width=\linewidth]{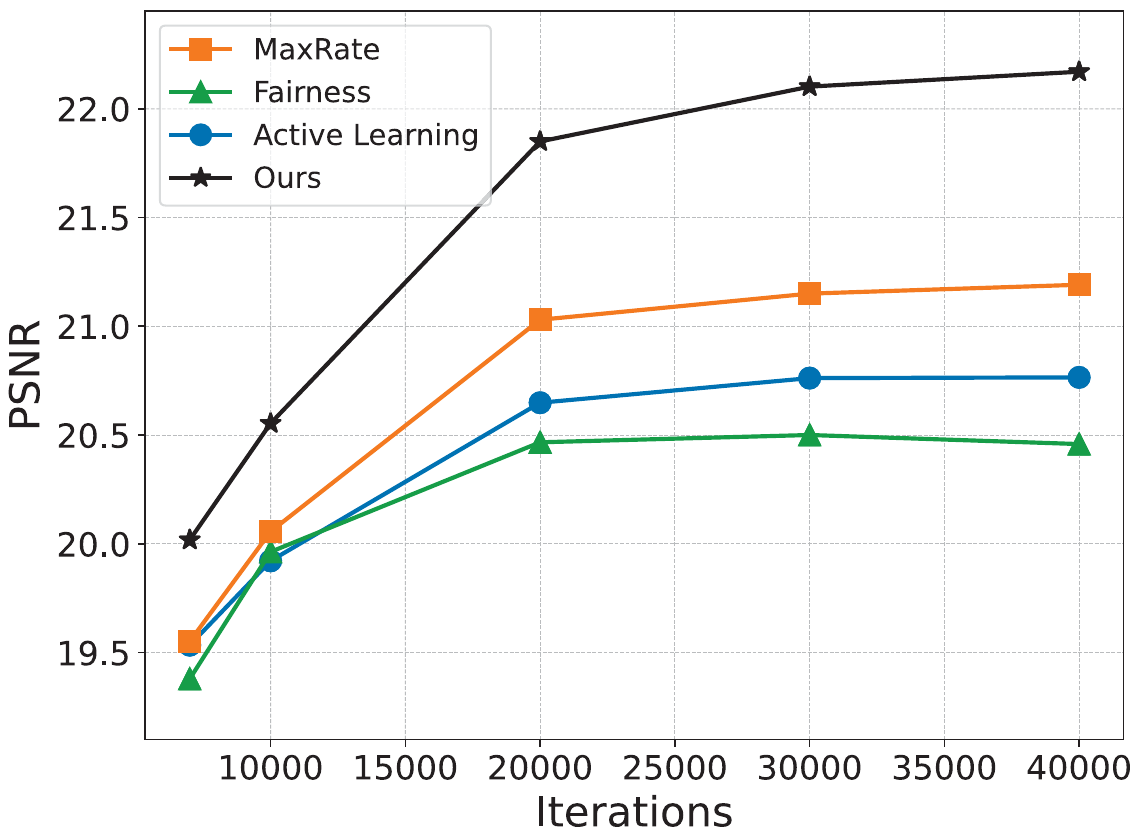}
        \caption{PSNR}
        \label{fig:psnr_performance_iterations}
    \end{subfigure}%
    \begin{subfigure}[b]{0.33\textwidth}
        \includegraphics[width=\linewidth]{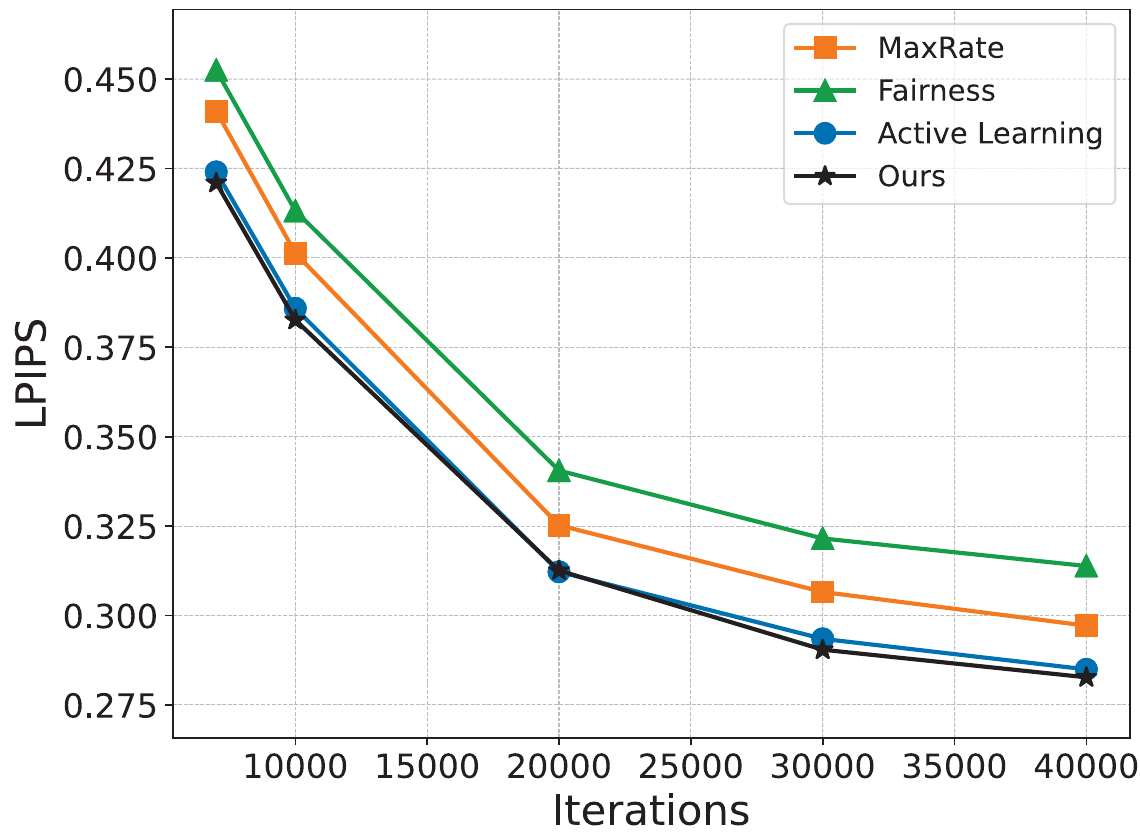}
        \caption{LPIPS}
        \label{fig:lpips_performance_iterations}
    \end{subfigure}%
    \begin{subfigure}[b]{0.33\textwidth}
        \includegraphics[width=\linewidth]{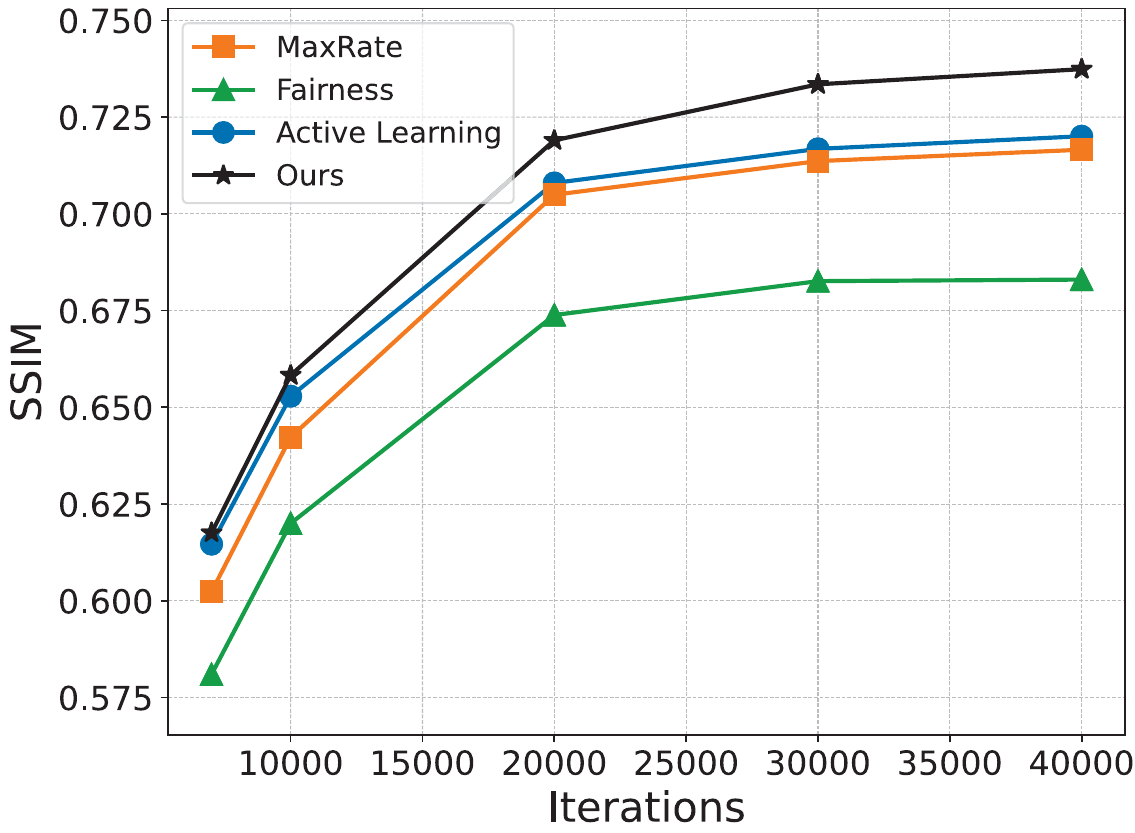}
        \caption{SSIM}
        \label{fig:ssim_performance_iterations}
    \end{subfigure}
    
    \caption{PSNR, LPIPS and SSIM versus the number of iterations.}
    \label{fig:performance_metrics_iterations}
\end{figure*}

\begin{figure}[!t]
    \centering
    \begin{subfigure}[b]{0.48\linewidth}
        \centering
        \includegraphics[width=\linewidth]{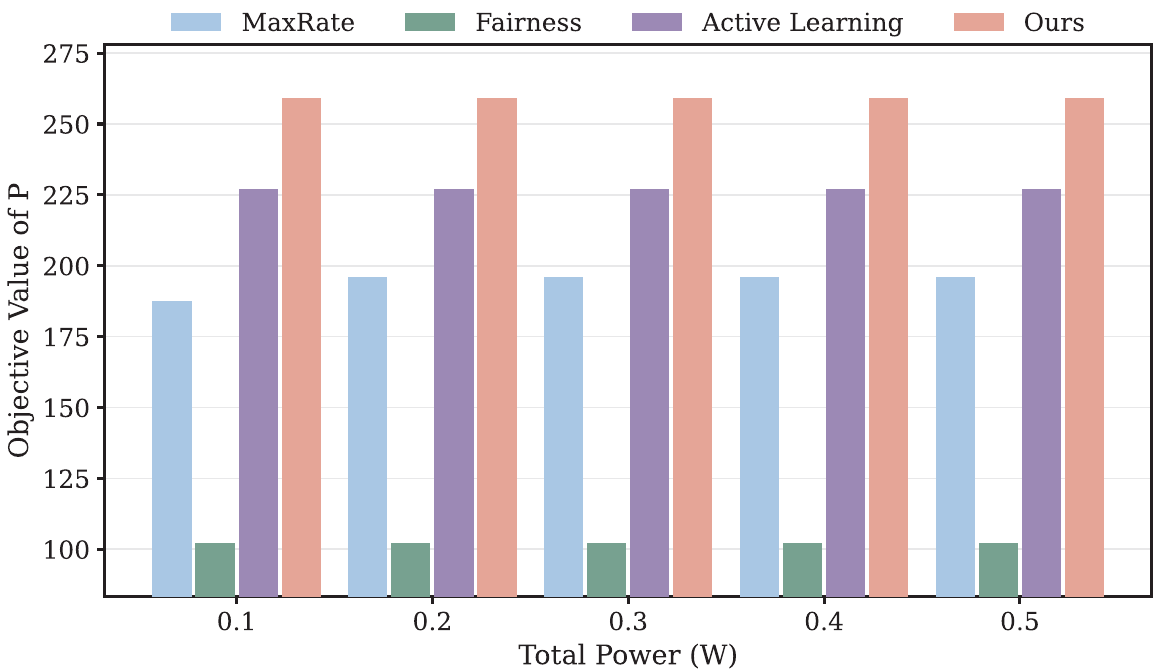}
        \caption{GS objective value}
        \label{fig:total_loss_power}
    \end{subfigure}
    \begin{subfigure}[b]{0.48\linewidth}
        \centering
        \includegraphics[width=\linewidth]{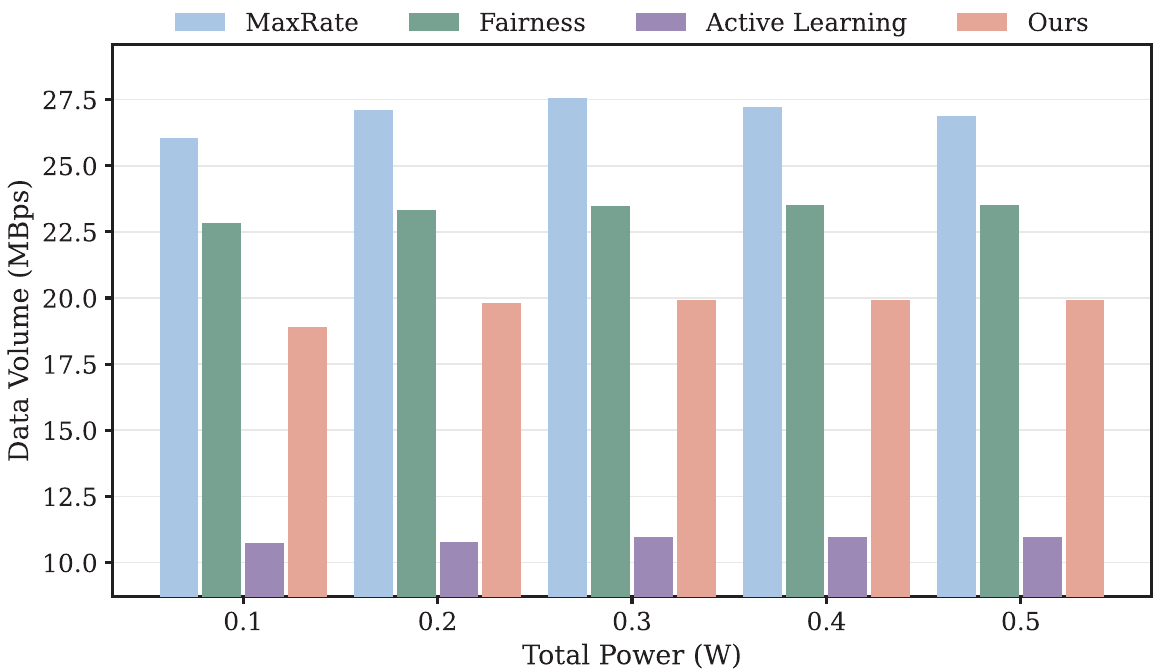}
        \caption{Data volume}
        \label{fig:sum_rate_power}
    \end{subfigure}
    \caption{Performance under different power budgets.}
    \label{fig:power_impact}
\end{figure}

\begin{figure}[!t]
    \centering
    \begin{subfigure}[b]{0.48\linewidth}
        \centering
        \includegraphics[width=\linewidth]{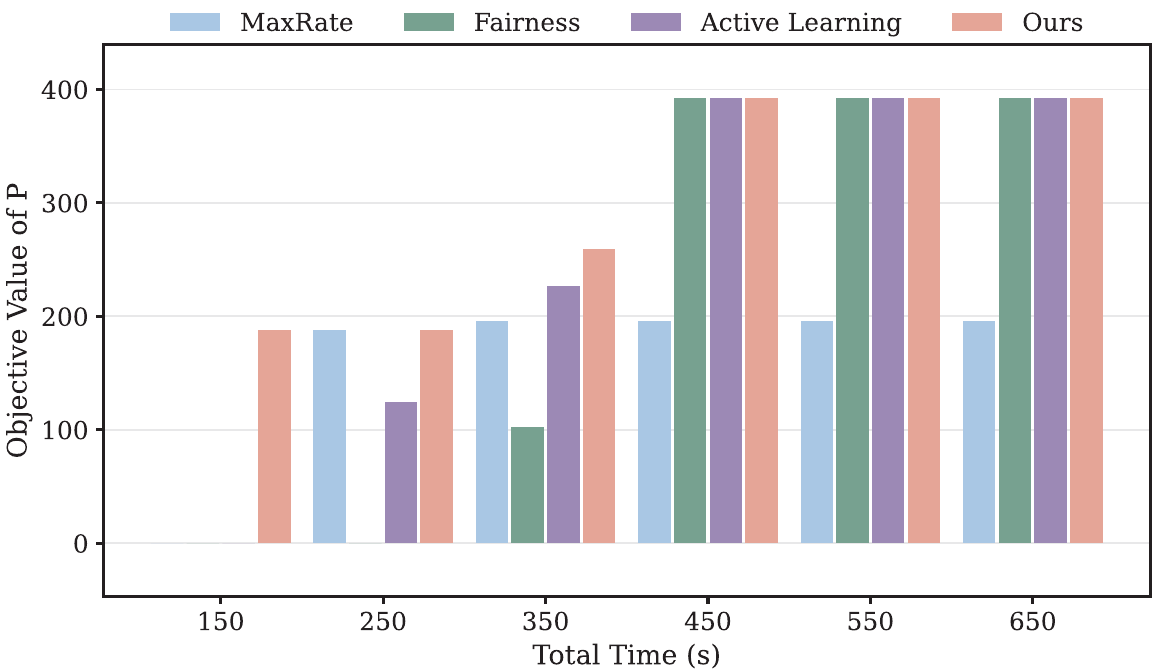}
        \caption{GS objective value}
        \label{fig:total_loss_time}
    \end{subfigure}
    \begin{subfigure}[b]{0.48\linewidth}
        \centering
        \includegraphics[width=\linewidth]{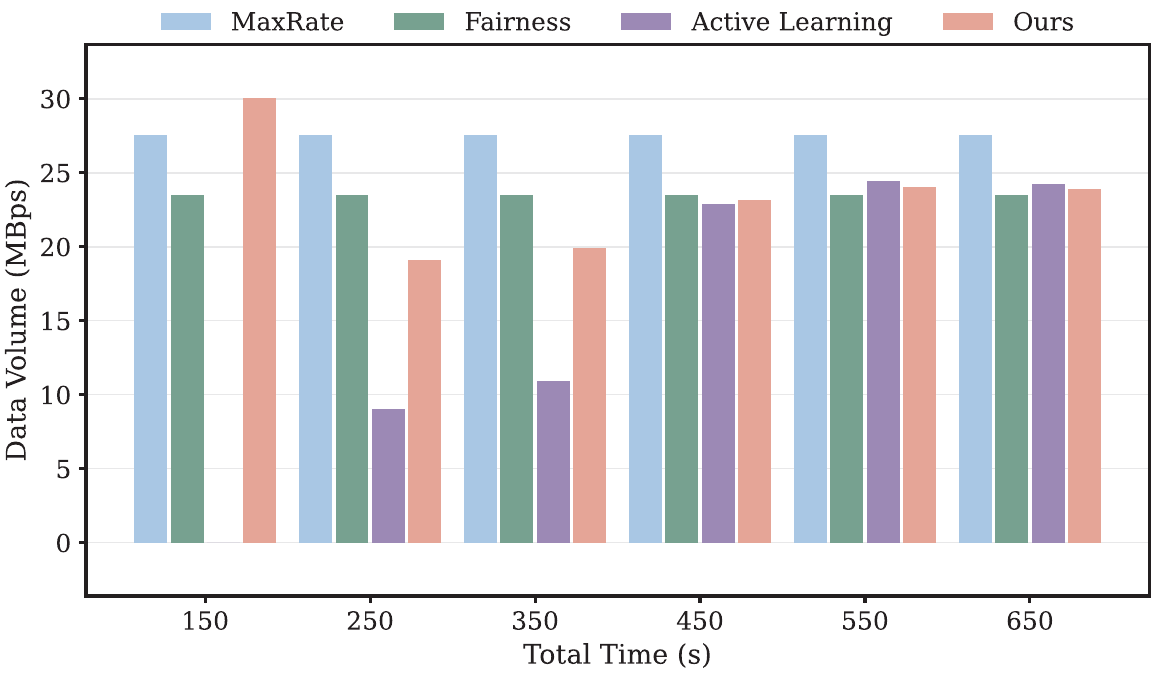}
        \caption{Data volume}
        \label{fig:sum_rate_time}
    \end{subfigure}
    \caption{Performance under different time budgets.}
    \label{fig:time_variation}
\end{figure}

\subsection{Sensitivity Analysis}

To test the robustness of the proposed method, we train the GS model under different training iterations, i.e., $(7000, 10000, 20000, 30000, 40000)$. The trained models are then used to render images, with performance evaluated using PSNR, LPIPS, and SSIM metrics. 
As can be observed in Fig.~\ref{fig:performance_metrics_iterations}, the PSNR and SSIM values increase as the number of iterations increases. 
Specifically, our STT-GS demonstrates superior performance compared to three other methods across all the iterations. 
We also observe that the rendering performance becomes saturated when the number of iterations exceeds $30,000$. 
This supports the rationale for a default setting of $30,000$ iterations for GS training, which corroborates the findings in \cite{kerbl20233d}.

Then, we evaluate the performance of the proposed and benchmark schemes under different values of power budgets and time budgets. 
The associated quantitative results are shown in Fig.~\ref{fig:power_impact} and Fig.~\ref{fig:time_variation}, respectively. 
It can be seen from Fig.~\ref{fig:power_impact}a and Fig.~\ref{fig:time_variation}a that no matter how the resource budget varies, the proposed method always achieves the largest objective value of $\mathsf{P}$. 
{\color{black}This corroborates our theory in Section II, and implies that our method collects those datasets that could change the GS model from $\mathcal{S}'$ to $\mathcal{S}$ with the maximum extent.}
Interestingly, it can be observed from Fig.~\ref{fig:power_impact}b and Fig.~\ref{fig:time_variation}b that the MaxRate scheme always collects the most data due to its design objective but performing poorer than our method. 
{\color{black}This demonstrates that our method focuses not only on the quantity but also on the quality of the data to be transmitted, thereby enhancing performance in the GS context.}

Finally, we evaluate our method on another LAE dataset, building-pixsfm. Table~\ref{Table-different} reports the PSNR, SSIM, and LPIPS metrics. 
Again, our method significantly enhances the rendering quality. Specifically, for PSNR, our method outperforms MaxRate by 5.25\% and Fairness by 2.58\%.
For SSIM, our method  outperforms MaxRate by 11.37\% and Fairness by 5.71\%.
For LPIPS, our method reduces the error by 10.54\% compared to MaxRate and 4.66\% compared to Fairness. Moreover, Fig.~\ref{fig:building_visual_comparison} presents rendering visualizations produced by different methods, demonstrating that the proposed approach achieves the highest rendering quality. These results demonstrate the robustness and strong generalization capability of our method across various scenarios.

\begin{figure}[t] 
    \centering
    \includegraphics[clip, width=0.98\columnwidth]{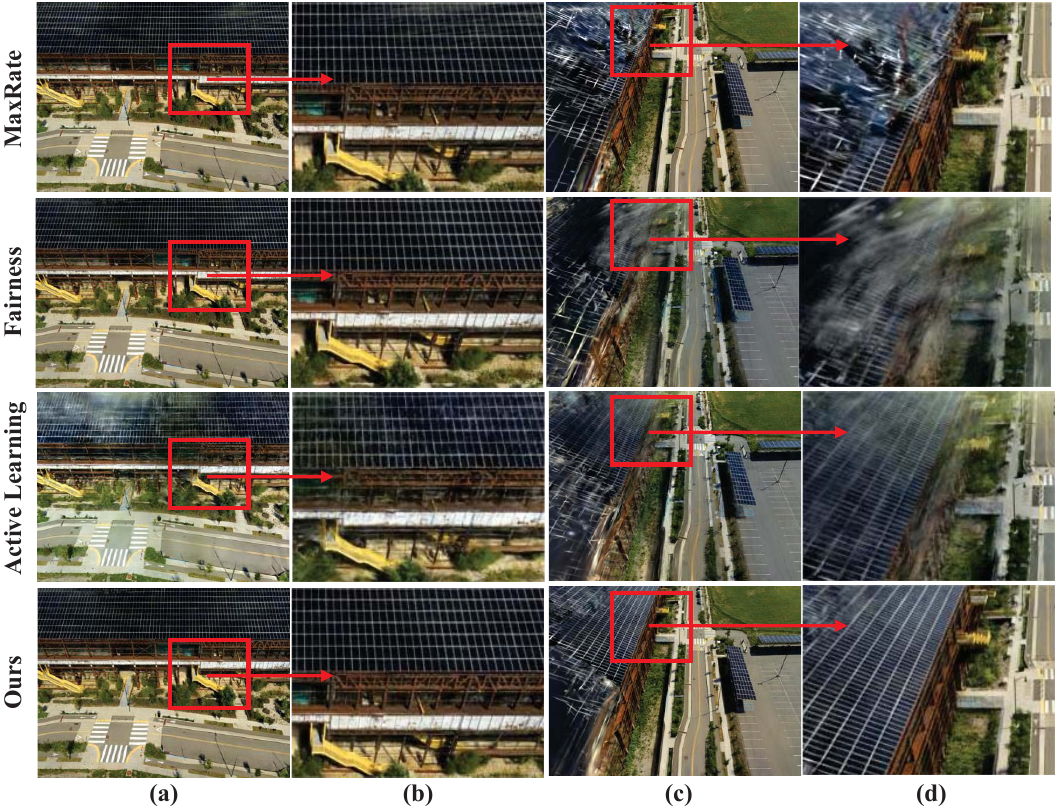}
    \caption{Comparison of rendering qualities on building-pixsfm.}
    \label{fig:building_visual_comparison}
\end{figure}

\begin{table}[!t]
\caption{Quantitative Results on Building-pixsfm Datesets}
\label{Table-different}
\centering
\scalebox{0.65}{
\begin{tabular}{lccc}
    \toprule
    & \multicolumn{3}{c}{Building-pixsfm}  \\
    \cmidrule(lr){2-4} 
    Method & $\uparrow$PSNR & $\uparrow$SSIM &$\downarrow$ LPIPS\\
    \midrule
    MaxRate & 17.2806 & 0.55589 & 0.39007 \\
    Fairness & 17.7297 & 0.58566 & 0.36600  \\
    Active Learning & 18.1186 & 0.61027 & 0.34410  \\
    Unrestricted & 19.0629 & 0.65454 & 0.32791 \\
    Ours & \textbf{18.1879}(+5.25\%) & \textbf{0.61910}(+11.37\%) & \textbf{0.34895} (-10.54\%) \\
    \bottomrule
    \vspace{-0.15in}
    \\
    \multicolumn{3}{c}{Note: Performance gain is computed against the MaxRate scheme} 
  \end{tabular}
}
\vspace{-0.05in}
\end{table}

\section{Conclusion}\label{section7}

This paper presented a novel STT-GS paradigm for multi-client 3D reconstruction. 
Our approach efficiently addressed the causality dilemma associated with optimizing an unknown GS-oriented objective function by prioritizing communication resources towards more valuable clients with higher view contributions. 
The FDC and PTTM algorithms were proposed to reduce pilot overhead and a joint optimization of client selection and power allocation was conducted for the EGS system based on PAMM. 
Our experiments demonstrated that our proposed scheme with FDC, PTTM, and PAMM algorithms outperforms various existing benchmarks.
It is found that serving clients with the most valuable GS data or the best channel condition is not always beneficial.
Interestingly, the proposed method effectively balances GS rendering and communication cost by a two-stage cross-layer optimization.

\bibliographystyle{IEEEtran}
\bibliography{ref}
\appendix
\section{Proof of Proposition 1}
\label{app:proof_prop1}

\subsection{Upper Bound Property}

{For $\widehat{\varphi}_1$:}  
Let $\varphi_1(\mathbf x)=\frac{1}{\beta}\sum_{k=1}^K x_{k}(1-x_k)$. The surrogate function is:
\begin{align*}\widehat{\varphi}_1(\mathbf x|\mathbf x^\star ) &= 
\sum_{k=1}^K 
\left(
\frac{1}{\beta}x_{k}-\frac{2}{\beta}x_{k}^{\star}x_{k}+\frac{1}{\beta}x_{k}^{\star^2}
\right).
\end{align*}
Subtracting $\varphi_1(\mathbf x)$ from $\widehat{\varphi}_1(\mathbf x|\mathbf x^\star)$:
\begin{align*}
\widehat{\varphi}_1(\mathbf{x}|\mathbf{x}^\star) - \varphi_1(\mathbf{x}) 
  &= \frac{1}{\beta} \sum_{k=1}^K \left(-x_k^2 + 2x_k^\star x_k - x_k^{\star 2}\right) \\
  &= \frac{1}{\beta} \sum_{k=1}^K (x_k - x_k^\star)^2 \geq 0.
\end{align*}
Thus, $\widehat{\varphi}_1(\mathbf x|\mathbf x^\star) \geq \varphi_1(\mathbf x)$.

{For $\widehat{\Phi}_k$:}  
Let $f(\bm{\xi}) = \log_2\left(1 + \frac{\xi_k H_{k,k}}{\sum_{j \neq k} \xi_j H_{k,j} + \sigma^2}\right)$. Since $f(\bm{\xi})$ is concave, its first-order approximation at $\bm{\xi}^\star$ satisfies:
$$
f(\bm{\xi}) \leq f(\bm{\xi}^\star) + \nabla f(\bm{\xi}^\star)^T (\bm{\xi} - \bm{\xi}^\star).
$$
Substituting into $\widehat{\Phi}_k$:
$$
\widehat{\Phi}_k(\mathbf{x},\bm{\xi}|\bm{\xi}^\star) = \eta_k x_k - \left[f(\bm{\xi}^\star) + \nabla f(\bm{\xi}^\star)^T (\bm{\xi} - \bm{\xi}^\star)\right] \geq \Phi_k(\mathbf{x},\bm{\xi}).
$$

\subsection{Convexity}

We compute the Hessian of $\widehat{\varphi}_1$ with respect to $x$ and 
the Hessian of $\widehat{\Phi}_k$ with respect to $(x,\bm{\xi})$. 
After standard mathematical calculations, it can be shown that 
\begin{align}
    \nabla_\mathbf x^2 \widehat{\varphi}_1 = \mathbf{0}, \ 
     \nabla_{(x,\bm{\xi})}^2 \widehat{\Phi}_k = \mathbf{0}, \ \forall k,
\end{align}
which are both positive semi-definite. Hence, $\widehat{\varphi}_1$ and $\widehat{\Phi}_k$ are convex.

\subsection{Local Equivalence}

\emph{Function Value}:  
At $(\mathbf{x}^\star, \bm{\xi}^\star)$:
$$
\widehat{\varphi}_1(\mathbf{x}^\star | \mathbf{x}^\star) = \frac{1}{\beta} \sum_{k=1}^K \left(x_k^\star - x_k^{\star 2}\right) = \varphi_1(\mathbf{x}^\star).
$$
Similarly, substituting $\bm{\xi} = \bm{\xi}^\star$ into $\widehat{\Phi}_k$, we have
$$
\widehat{\Phi}_k(\mathbf{x}^\star, \bm{\xi}^\star | \bm{\xi}^\star) = \Phi_k(\mathbf{x}^\star, \bm{\xi}^\star).
$$

\emph{Gradient Value}:  
For $\widehat{\varphi}_1$:
$$
\nabla_\mathbf{x} \widehat{\varphi}_1(\mathbf{x} | \mathbf{x}^\star) = \frac{1}{\beta} \left(1 - 2\mathbf{x}^\star\right), \ \nabla_\mathbf{x} \varphi_1(\mathbf{x}) = \frac{1}{\beta} \left(1 - 2\mathbf{x}\right).
$$
Therefore, $\nabla_\mathbf{x} \widehat{\varphi}_1(\mathbf{x} | \mathbf{x}^\star)=
\nabla_\mathbf{x} \varphi_1(\mathbf{x})
$ at $\mathbf{x} = \mathbf{x}^\star$. 
Similarly, it can be verified that  $\nabla_{(\mathbf{x}, \bm{\xi})} \widehat{\Phi}_k$ equals $\nabla_{(\mathbf{x}, \bm{\xi})} \Phi_k$ at solution point $(\mathbf{x}^\star, \bm{\xi}^\star)$.
\end{document}